\documentclass[letterpaper,twocolumn,10pt]{article}
\usepackage{usenix2019_v3}

\usepackage{amsmath,amssymb,amsfonts,amsthm}
\usepackage{graphicx}
\usepackage{booktabs}
\usepackage{tabularx}
\usepackage{array}
\usepackage{float}
\usepackage{listings}
\usepackage{algorithm}
\usepackage{algpseudocode}
\usepackage{subcaption}
\usepackage{xspace}
\usepackage{tikz}
\usetikzlibrary{arrows.meta, positioning, calc, shapes.geometric, fit, shadows, matrix}
\microtypesetup{spacing=false}

\AtBeginDocument{\DeclareMathAlphabet{\mathcal}{OMS}{cmsy}{m}{n}}

\newcommand{\tct}{\mbox{TCT}\xspace}
\newcommand{\esr}{\mbox{ESR}\xspace}

\newtheorem{theorem}{Theorem}
\newtheorem{lemma}{Lemma}
\newtheorem{definition}{Definition}

\newtheorem{proposition}{Proposition}
\newtheorem{corollary}{Corollary}

\definecolor{slate}{RGB}{112,128,144}
\definecolor{emerald}{RGB}{80,200,120}
\definecolor{navy}{RGB}{20,50,110}
\definecolor{crimson}{RGB}{180,30,30}

\hypersetup{
  pdftitle={When AI Agents Commit: Cognitive Serializability Across Data, Evidence, Policy, and Authority},
  pdfauthor={Jun He; Deying Yu},
  pdfsubject={Transaction correctness for mutations derived by non-deterministic reasoning},
  pdfkeywords={cognitive transactions, serializability, agentic systems, dependency fencing, PostgreSQL, commit-time authorization}
}

\begin{document}

\title{\bf When AI Agents Commit: Cognitive Serializability Across Data, Evidence, Policy, and Authority}

\author{
  {\rm Jun He}\\
  OpenKedge.io
  \and
  {\rm Deying Yu}\\
  OpenKedge.io
}

\maketitle

\begin{abstract}
Autonomous agents derive concrete mutations from database reads, retrieved
evidence, policy, beliefs, and delegated authority. Those inputs may change
while reasoning is in progress. Database isolation orders the submitted
transaction; agentic transaction processing determines whether a proposal
satisfies an executable contract. Neither guarantee establishes a common valid
point for the mutation and its derivation inputs unless the contract represents
the relevant predicates. Typed dependency tokens distinguish content integrity
from applicability, and trusted mediation captures the values exposed to
reasoning. Under \emph{strict Cognitive Serializability}, committed effects
admit a serial order
and a logical event at which every value exposed to derivation is unchanged.
The fences last until the runtime event that realizes the sealed durability
domain. The weaker \emph{Effect-Compatible Cognitive Admission} recertifies an
effect against a simultaneously held current
dependency vector and current policy without claiming to serialize the
original stochastic derivation. TCT combines immutable versioned executable
definitions, registry-derived authority plans, sealed envelopes, guard-first
commit transactions, post-seal envelope- and witness-bound grants,
co-committed receipts, idempotent grant finalization, and receipt-driven
epistemic reconciliation. Complete registered footprints and a single growing
phase induce an acyclic lock-point order over local guards and incompatible
external reservations. The corresponding results give serializability
conditions and an observational-equivalence boundary for zero-error soundness
and positive progress. A falsification suite tests the implementation
obligations: the prototype prevented all injected anomalies and added
3.22~ms mean commit overhead.
\end{abstract}

\section{Introduction}
\label{sec:intro}

AI agents increasingly propose high-consequence mutations after interpreting
unstructured data, retrieving evidence, evaluating policy, and exercising
delegated authority. Unlike application logic that reads and writes within one
ACID transaction, an agent may reason for seconds before submitting an
operation. Its database, evidence, policy, or authority inputs may change in
the meantime. We call the resulting operations \emph{non-deterministic
mutations}: deterministic state changes derived by non-deterministic reasoning
over volatile inputs.

A serializable database commit may therefore be valid inside the DBMS but no
longer justified by the state used to derive it. Mechanisms such as PostgreSQL
\texttt{SERIALIZABLE}~\cite{ports2012ssi} order database operations, not remote
evidence selection, policy epochs, or external authorization premises.
Database serializability alone cannot identify a point at which the effect and
all derivation premises were simultaneously valid.

Agentic Transaction Processing (ATP) addresses proposal admission by
passing generated proposals through deterministic admission gates
~\cite{mnemosyne2026}. Its guarantees are relative to predicates represented
in constraint set $\mathcal C$; they do not by themselves establish complete
input capture, correct external-validity semantics, or protection through the
target durability event.

Bridging stochastic derivation and deterministic commit requires three
obligations: trusted capture of every exposed input as a typed dependency;
separate semantics for content, applicability, and commit-time authority; and
validation over one state protected through durability. We ask: \emph{when
does a mutation derived outside an ACID transaction admit a common
serialization point with the database, evidence, policy, and authority states
on which it depends?}

\textbf{Taming Cognitive Transactions} (\tct) is a typed dependency contract
for an admission-and-commit gate. Its strict profile, \emph{Cognitive
Serializability} (\tct-S), provides a database serial order and a logical event
at which every mediated derivation value remains unchanged. Its compatible
profile, \emph{Effect-Compatible Cognitive Admission} (\tct-C), recertifies
the effect against a simultaneously protected current dependency vector and
policy. TCT-C may admit after dependency drift, but it does not serialize the
original derivation. The distinction prevents effect recertification from
being read as a claim about derivation history.

TCT realizes this distinction through immutable executable definitions, a sealed envelope containing the registered footprint and authority plan, guard-first local locking, and four mechanisms for preserving external validity through a sealed durability domain. Post-seal grants bind the envelope, effect, covered dependency vector, and durability target; a co-committed receipt links the admission decision to the durable effect and drives idempotent grant finalization and epistemic reconciliation. Under stated capture, registry, writer-participation, external-validity, and durability assumptions, the strict profile's single growing phase induces an acyclic lock-point order (Theorem~\ref{thm:cognitive_serializability}). TCT can be embedded in an ATP runtime when $\mathcal C$ implements the complete contract (Proposition~\ref{prop:atp_embedding}). The present work establishes a formal contract, sufficient conditions, PostgreSQL-oriented pseudocode, a 28-history falsification suite, and an empirical evaluation of a prototype gate demonstrating 100\% safety conformance and scalability up to 128 worker threads. Its intended deployment scope is high-consequence agentic mutation, where an unjustified durable effect warrants coordination overhead.

\paragraph{Contributions.}
The paper makes four contributions:

\begin{enumerate}
    \item \textbf{Two distinct correctness guarantees.} TCT-S provides strict, derivation-faithful Cognitive Serializability: every mediated value used in reasoning must remain valid at a common logical point. TCT-C provides the weaker Effect-Compatible Cognitive Admission guarantee: after inputs change, the concrete effect may commit only if a registered joint predicate accepts the complete, simultaneously protected current dependency vector and policy.
    \item \textbf{A closed-input transaction contract.} Trusted components
    capture inputs, resolve operation and policy requirements, and seal the
    contract before admission. It separates content identity from selection
    and applicability, and records the footprint, dependency discharges,
    policy reference, correctness profile, and authority plan.
    \item \textbf{A guard-first commit and reconciliation protocol.} The protocol acquires local guards before protected reads, preserves external validity through the target durability event using mechanisms A--D, and binds post-seal grants to the sealed envelope. Co-committed receipts then support safe retries, idempotent grant consumption, and receipt-driven epistemic reconciliation.
    \item \textbf{Formal and empirical validation.} We state sufficient conditions for acyclic TCT-S histories and effect--receipt correspondence, show when TCT composes with an ATP host, and identify both constraint-relative non-implication and an observational-equivalence boundary on zero-error admission. A 28-history falsification suite evaluates whether the prototype enforces these conditions under controlled perturbations.
\end{enumerate}

\paragraph{Roadmap.}
Section~\ref{sec:motivation} motivates the problem and reviews the technical foundations, and Section~\ref{sec:system-model} defines the model. Sections~\ref{sec:cognitive-serializability}--\ref{sec:theorems} state the correctness conditions and formal results. Section~\ref{sec:legacy} presents the PostgreSQL realization, and Section~\ref{sec:evaluation} evaluates its safety conformance and operational cost. Sections~\ref{sec:related}--\ref{sec:conclusion} position the work, state its limitations, and conclude.

\section{Motivation and Technical Foundations}
\label{sec:motivation}

Consider a supply-chain agent that reads inventory, queries vendor
certification, applies spending policy, and checks management approval before
submitting order proposal $\mu$. Its inputs span four domains:
\begin{itemize}
    \item \textbf{Database state ($\sigma_D$):} inventory row versions, version tokens, and aggregate projections.
    \item \textbf{Evidence ($J_E$):} immutable content bytes of retrieved artifacts (e.g., a vendor certificate) paired with an active issuer selection determining whether that certificate remains currently applicable for the order.
    \item \textbf{Policy ($J_P$):} active policy rules, compliance constraints, and temporal spending-policy epoch bounds.
    \item \textbf{Observed authority ($J_{A\mathrm{obs}}$):} derivation-time capability observations and approval premises shown to the reasoner. A trusted verifier $Q_X$ separately determines mandatory post-seal authority obligations at commit time.
\end{itemize}

\subsection{The Epistemic-to-ACID Gap}

No single transactional guarantee bridges non-deterministic derivation and
durable effect. Four boundaries matter:
\begin{description}
    \item[Isolation scope.] Standard database serializability orders reads and writes on managed database tuples. Remote evidence artifacts, policy rule sets, and external authority decisions lie entirely outside the DBMS lock manager and validation engine.
    \item[Contract completeness.] A deterministic admission gate can only validate what its executable constraint set $\mathcal C$ represents. Any operational dependency omitted from $\mathcal C$ escapes verification entirely.
    \item[Validity horizon.] Point-in-time validation does not guarantee premise stability through physical commit. Content immutability does not imply selection applicability, and derivation-time observation does not establish commit-time authorization.
    \item[Epistemic settlement.] Unconfirmed database commits, serialization aborts, or outbox delivery failures must not be interpreted by the agent as durable effects. Belief activation requires receipt-gated reconciliation.
\end{description}

\tct combines mediated capture, typed validity semantics, commit-spanning
fences, and receipt-driven epistemic reconciliation across these boundaries.

\subsection{Cross-State Anomalies}

The supply-chain scenario exhibits four corresponding failure modes when operational inputs drift prior to commit:

\begin{description}
    \item[Stale-snapshot derivation ($\sigma_D$).] A concurrent transaction modifies an inventory row or aggregate version used to derive $\mu$. Traditional database lock managers admit $\mu$ if its write set is disjoint from the concurrent write, even though the proposed quantity is no longer valid under the updated database state.
    \item[Unfenced evidence drift ($J_E$).] The observed certificate content bytes remain static and hash-valid, but the issuing authority revokes or supersedes the certificate during derivation. A digest or one-time freshness check passes, yet the current applicability premise fails at commit time.
    \item[Revoked-authority commit ($J_{A\mathrm{obs}}$).] An approval premise or capability observed during derivation is revoked, expires, or becomes ineligible before durability, while the static database role executing the transaction remains authorized to write.
    \item[Premature belief activation ($B$).] The agent activates an internal belief update $\Delta B$ (for example, ``order placed'') before confirming the database outcome. A subsequent serialization abort or network timeout leaves a phantom belief.
\end{description}

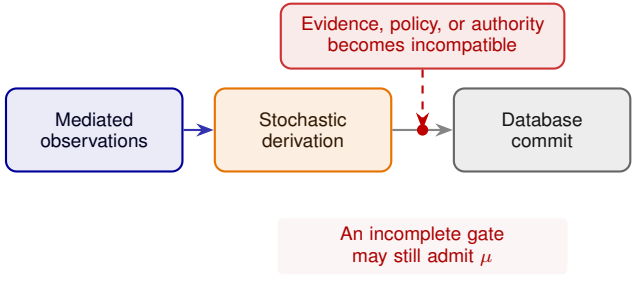
\begin{figure}[t]
\centering
\colorlet{mynavy}{blue!60!black}
\colorlet{myorange}{orange!90!black}
\colorlet{myslate}{black!60}
\colorlet{mycrimson}{red!75!black}

\begin{tikzpicture}[
    font=\scriptsize\sffamily,
    >=Stealth,
    process/.style={
        thick, rounded corners=4pt, align=center,
        text width=2.1cm, minimum height=1.1cm
    },
    gap node/.style={
        circle, fill=mycrimson, inner sep=1.5pt
    }
]

    \node (obs) [process, draw=mynavy, fill=mynavy!8, text=mynavy!20!black]
        {Mediated\\observations};

    \node (reason) [process, draw=myorange, fill=myorange!12, text=myorange!20!black,
        right=4mm of obs]
        {Stochastic\\derivation};

    \node (commit) [process, draw=myslate, fill=myslate!12, text=myslate!20!black,
        right=8mm of reason]
        {Database\\commit};

    \draw[->, thick, mynavy!80] (obs) -- (reason);
    \draw[->, thick, myslate!80] (reason) -- (commit);

    \coordinate (gap) at ($(reason.east)!0.5!(commit.west)$);

    \node (drift) [
        draw=mycrimson!80, thick, fill=mycrimson!8, rounded corners=4pt,
        text=mycrimson!90!black, align=center, above=8mm of gap,
        text width=3.4cm, inner sep=5pt
    ] {Evidence, policy, or authority\\becomes incompatible};

    \draw[->, dashed, mycrimson, thick] (drift.south) -- (gap);
    \node[gap node] at (gap) {};

    \node (warning) [
        fill=mycrimson!4, rounded corners=2pt,
        text=mycrimson!90!black, align=center, text width=3.6cm,
        below=6mm of $(reason.south)!0.5!(commit.south)$
    ] {An incomplete gate\\may still admit $\mu$};

\end{tikzpicture}
\caption{The external-validity portion of the epistemic-to-ACID gap. A proposal
gate can detect the illustrated drift when its executable contract represents
the dependency and supplies a validity fence. The anomaly arises when that
contract omits the dependency or checks it without a fence that lasts through
commit.}
\label{fig:gap_anomaly}
\end{figure}

Figure~\ref{fig:gap_anomaly} locates the failure at the contract boundary: detection requires both complete dependency representation and fences that remain valid through physical commit.

\subsection{Technical Foundations}
\label{sec:background}

The \tct transaction contract synthesizes foundations from database isolation, epistemic state management, agentic admission, commit-time authorization, and tool-effect settlement.

\subsubsection{Epistemic State Replication}

Epistemic State Replication (\esr)~\cite{esr2026} separates an immutable
evidence log $L$ from stochastic belief state $B$:
\begin{equation}
K=(L,B).
\end{equation}
$L$ records evidence; $B$ records derived beliefs and lineage. \tct uses ESR
downstream: commit receipts gate belief activation, while ESR may still reject
conflicting deltas.

\subsubsection{Database Serializability and External Consistency}

Classical serializability orders committed database histories
~\cite{bernstein1987concurrency,gray1992transaction,adya1999generalized};
PostgreSQL SSI aborts dangerous dependency cycles~\cite{ports2012ssi}; strict
serializability also respects real time~\cite{corbett2013spanner}. \tct instead
defines a fenced logical event spanning database and external premises.

\subsubsection{Agentic Concurrency and Coordination Systems}

Recent systems address complementary parts of agentic transaction processing:
\begin{itemize}
    \item \textbf{ATP/Mnemosyne} validates generated proposals against current
    state under $\mathcal C$ and supplies SEA, transition logging, repair, and
    outbox staging~\cite{mnemosyne2026}. Its guarantees are relative to
    predicates in $\mathcal C$.
    \item \textbf{Commit-Time Authorization} enforces freshness, causal priority, exact effect binding, and eligibility at a durable-effect boundary, implementing a fail-closed CommitGuard when runtimes supply required signals~\cite{commit_authorization2026}.
    \item \textbf{ATCC} adapts optimistic and pessimistic control across long
    agentic transactions~\cite{atcc2026}.
    \item \textbf{CoAgent and S-Bus} respectively repair concurrent
    trajectories and reconstruct HTTP-observable read sets
    ~\cite{coagent2026,sbus2026}.
    \item \textbf{Atomix} seals read/effect scopes, waits for resource
    frontiers, and settles effects by reversibility class~\cite{atomix2026}.
\end{itemize}

\paragraph{Implication for TCT.}
These foundations supply different parts of the contract: capture records
typed dependencies, locks and external fences protect them through commit, and
receipts govern epistemic reconciliation. Section~\ref{sec:related} gives the
detailed comparison.

\section{System Model, Dependency Capture, and Sealing}
\label{sec:system-model}

The model separates authoritative database state $D$, evidence $E$, policy
$P$, authority $A$, and epistemic state $K=(L,B)$. Execution separates trusted
capture, stochastic proposal generation, and trusted operation/policy
resolution.

\subsection{Capture, Reasoning, and Policy Resolution}

Before capture, trusted context
$c=\langle tenant,operation\_class,\ldots\rangle$ fixes an operation class and
the mandatory derivation policy shared by every operation selectable in that
class. Its canonical digest $h_c$ and policy reference are represented in
$J_P$. Trusted mediators expose operational inputs $U$ and produce mediated
intent and epistemic projections, their conservative typed closure, and the
policy actually exposed during derivation:
\begin{equation}
\begin{aligned}
  \mathsf{Capture}(U,c)
    & =\langle I_J,K_J,J_{\mathrm{obs}},
       \Pi_{\mathrm{obs}},r_{\Pi,o}\rangle,\\
  J_{\mathrm{obs}}
    & =J_D\uplus J_E\uplus J_P\uplus J_{A\mathrm{obs}}.
\end{aligned}
\label{eq:capture_function}
\end{equation}
$r_{\Pi,o}$ is captured by trusted mediation, its version and digest are bound
in a token $\delta_{\Pi,o}\in J_P$, and
$\Pi_{\mathrm{obs}}=\mathsf{Load}(r_{\Pi,o})$. Thus
$\Pi_{\mathrm{obs}}$ is the derivation-time policy (equivalently
$\Pi_{\mathrm{der}}$), not a policy discovered after reasoning. $I_J$ and
$K_J$ are canonical projections of mediated tokens in $J_{\mathrm{obs}}$:
$K_J$ includes every
belief premise and cached-memory value made available to the reasoner, and
every field of either projection carries token provenance. The reasoner
receives no direct unqualified $I$, $K$, cache, file, tool result, or other
operational input.
The stochastic reasoner proposes an operation identifier, typed parameters, a
tentative belief delta, and optional strengthening constraints:
\begin{equation}
 \mathsf{Reason}_{\theta}
 (I_J,K_J,\mathsf{View}(J_{\mathrm{obs}}),\Pi_{\mathrm{obs}})
 \rightsquigarrow (op,p,\Delta B,\Pi^{+}).
\label{eq:reasoning_function}
\end{equation}
It produces neither $J_{\mathrm{obs}}$, authoritative policy, nor a mandatory
grant plan. At sealing, trusted infrastructure resolves the operation
implementation, derives the concrete effect, and constructs the complete
authority and external-validity plan:
\begin{equation}
\mathsf{ResolveSeal}(op,p,c,J_{\mathrm{obs}},\Pi^{+})
 =\langle\mu,\Omega_{\mu},Q_X,m_X\rangle .
\label{eq:seal_resolution}
\end{equation}
Trusted mediation supplies context $c$; the proposer cannot. Resolution
rejects unless $op$ belongs to the sealed operation class governed
by $r_{\Pi,o}$; it cannot substitute a newly resolved policy for the one exposed
during reasoning.
The immutable operation specification
\[
\begin{aligned}
\Omega_{\mu}=\langle
 &op,\nu_{\Omega},h_{\Omega},R_D(\mu),W_D(\mu),\\[-1mm]
 &r_{\mathrm{pre}},r_{\mathrm{dispatch}},r_{\mathrm{canon}},
 \mathsf{family}_{\Phi}\rangle
\end{aligned}
\]
derives the sealed effect, a conservative complete local footprint, its
versioned precondition and dispatcher, its canonicalizer, and the permitted
joint-recertifier family.
The trusted profile $m_X\in\{S,C\}$ is derived from the operation class,
mandatory derivation policy, and any permitted strengthening in $\Pi^{+}$; it
is not chosen by the caller. Profile compatibility is exact equality. Neither
cross-profile direction is generally valid: TCT-S is not assumed to satisfy an
arbitrary TCT-C recertifier, and TCT-C does not satisfy TCT-S's derivation
faithfulness requirement.

The plan is the canonical ordered vector $Q_X=\langle q_j\rangle_{j=1}^{m}$,
with $h_{Q_X}=\mathsf{Hash}(\mathsf{Canon}(Q_X))$, where
\begin{equation}
\begin{aligned}
q_j=\langle&
\iota_j,\mathit{scope}_j,m_j,\mathsf{Cover}(q_j),r_{w,j},\\[-1mm]
&r_{\chi,j},\mathit{decision}_j,\mathcal D_X,M_j\rangle .
\end{aligned}
\label{eq:authority_plan_item}
\end{equation}
Here $\iota_j$ identifies the required issuer or atomic-commit participant;
$m_j\in\{S,X\}$ is its compatibility mode;
$\mathsf{Cover}(q_j)\subseteq J_{\mathrm{obs}}$ is the ordered set of covered
dependency tokens; $r_{w,j}$ is a versioned witness-construction reference;
$r_{\chi,j}$ identifies validity semantics; $\mathit{decision}_j$ is the
required decision; $\mathcal D_X$ identifies the target durability domain and
its semantics; and $M_j\in\{A,B,C,D\}$ selects the validity mechanism.
$\mathcal D_X$ includes the specific gate and database/commit-domain identity,
so it cannot name an interchangeable durability class.

The trusted operation registry derives a mandatory plan
$Q_X^{\mathrm{mand}}$ from $(\Omega_{\mu},p,c,J_{\mathrm{obs}})$. Optional
authority constraints or grant requests originating in $p$ or elsewhere
outside that registry are admitted only through a monotone strengthening
$Q_X=Q_X^{\mathrm{mand}}\sqcup Q_X^{+}$: they may add obligations or tighten a
registered obligation, but cannot omit, replace, or weaken any member of
$Q_X^{\mathrm{mand}}$. Resolution rejects a non-monotone proposal.

Plan well-formedness also requires the converse of
$\mathsf{Cover}(q)\subseteq J_{\mathrm{obs}}$. Let
$\mathsf{SelfStable}(\delta)$ mean that the verified validity semantics for
$\delta$ proves the exact premise used in derivation immutable or otherwise
self-stable through any declared durability event; content immutability does
not make a current-selection premise self-stable. Let
$\mathsf{LocalA}(\delta,Q_X)$ mean that a registered co-transactional
mechanism-A rule names the dependency's complete guard and writer discipline.
Let $\mathsf{CoversThroughD}(\delta,q)$ mean that
$\delta\in\mathsf{Cover}(q)$, $M(q)\in\{B,C,D\}$, and the typed plan item binds
a realizable contract that protects that dependency through the event declared
by its sealed $\mathcal D_X$. Define the checkable structural relation
\[
\begin{gathered}
\mathsf{DepCovered}(J_{\mathrm{obs}},Q_X)\equiv\\[-1mm]
\forall\delta\in J_{\mathrm{obs}}:\quad
\mathsf{SelfStable}(\delta)\lor
\mathsf{LocalA}(\delta,Q_X)\\[-1mm]
{}\lor\exists q\in Q_X:
\mathsf{CoversThroughD}(\delta,q).
\end{gathered}
\]
The trusted registry checks this relation from typed, versioned definitions;
at admission, $\mathsf{MechanismThroughD}$ verifies that every selected
non-self-stable discharge actually holds through the realized event $d$.
Individually well-formed plan subsets are therefore insufficient unless every
captured dependency has a valid discharge. $\mathsf{ResolveSeal}$ succeeds only when
$\mathsf{DepCovered}(J_{\mathrm{obs}},Q_X)$ holds.

After the policy guard or external policy fence is held at admission, trusted
infrastructure resolves
\begin{equation}
\mathsf{ResolveCommit}(op,c,S_P(s))
 =\langle\Pi_{\mathrm{commit}},r_{\Pi,c},r_{\Phi}\rangle .
\label{eq:commit_resolution}
\end{equation}
$r_{\Pi,c}$ identifies the current mandatory bundle and
$r_{\Phi}=\langle\mathit{id}_{\Phi},\nu_{\Phi},h_{\Phi}\rangle$ identifies the
versioned joint recertifier selected by $(r_{\Omega},r_{\Pi,c})$. The effective
commit bundle is
\begin{equation}
  \Pi_{\mathrm{eff}}=\Pi_{\mathrm{commit}}\land\Pi^{+}.
\label{eq:effective_policy}
\end{equation}
$\Pi^{+}$ is evaluated only in a trusted constraint language; conjunction can
strengthen but cannot replace the current mandatory bundle. The proposer never
selects an authoritative policy epoch.

The protected current policy may also determine authority or external-validity
obligations. Under the same guards, trusted infrastructure therefore computes
\begin{multline*}
Q_{\mathrm{req}}(s)=
\mathsf{ResolvePlan}\bigl(
 \Omega_{\mu},p,c,J_{\mathrm{obs}},
 \Pi_{\mathrm{commit}},\\
 \Pi^{+},S_{J_{\mathrm{obs}}}(s)\bigr).
\end{multline*}
This resolver includes both current mandatory obligations and the optional
strengthening constraints already represented in the sealed proposal. TCT
uses the conservative exact rule
\[
 \mathsf{PlanCurrent}(Q_X,s)\equiv
 \mathsf{Canon}(Q_{\mathrm{req}}(s))=\mathsf{Canon}(Q_X).
\]
At sealing, $\mathsf{ResolveSeal}$ evaluates the same versioned plan rules
against the captured derivation policy and observed token values. At admission,
any current-policy or current-state change that would alter the required plan
causes abort and resealing. The gate never appends, replaces, or weakens a plan
item after $h_X$ exists. This check applies to both profiles and is essential
when TCT-C admits under a $\Pi_{\mathrm{commit}}$ different from
$\Pi_{\mathrm{obs}}$.
Under the same protection the gate computes
\[
m_{\mathrm{req}}(s)=\mathsf{ResolveProfile}
(\Omega_\mu,c,\Pi_{\mathrm{commit}},\Pi^+,S_{J_{\mathrm{obs}}}(s)).
\]
Admission requires
\[
  \mathsf{ProfileCurrent}(m_X,s)\equiv m_{\mathrm{req}}(s)=m_X.
\]
A mismatch rejects before the dispatcher or receipt write and requires
resealing.

\subsection{Executable Definitions and Registry Applicability}

Every executable registry object---an operation specification, precondition
evaluator, dispatcher, canonicalizer, validity-semantics definition, per-token
predicate, witness constructor, plan resolver, or envelope-level
recertifier---has a typed reference
\[
r_f=\langle kind_f,id_f,\nu_f,h_f,u_f\rangle .
\]
$h_f$ digests canonical definition bytes. The reference either embeds those
bytes or names a content-addressed object whose immutability and permanent
retention are part of the trusted registry contract. The small bootstrap loader
that verifies $r_f$ is fixed by the TCT protocol version. The envelope binds the
complete set $\mathcal R_X$ of executable references it may invoke. This set
includes the versioned trusted atomic normalizer $r_{\mathrm{norm}}$. For each
B/C plan item,
\[
\begin{aligned}
\mathsf{Atoms}_{r_{\mathrm{norm}}}(q_j)=
\langle&\mathit{coordinationDomain},\\[-1mm]
&\mathit{atomicLockId},m_j,j\rangle^{*}
\end{aligned}
\]
is its canonical vector of atomic acquisition identities, required modes, and
covered plan ordinals. Sealing merges duplicate identities with $X$ dominating
$S$ and binds
$h_{\mathsf{Atoms},X}=\mathsf{Hash}(\mathsf{Canon}(
\mathsf{MergeAtoms}_{r_{\mathrm{norm}}}(Q_X)))$.
$r_{\mathrm{norm}}\in\mathcal R_X$ and $h_{\mathsf{Atoms},X}$ are therefore
both envelope-bound. Admission recomputes the vector and requires this digest.
All normalizer versions admitted concurrently for a coordination domain must
assign the same identity to the same physical lock. An upgrade that changes
that mapping may become admissible only through a transition protected by the
discovery/normalizer-epoch guard. The transition first stops admitting
envelopes sealed under the old epoch, then waits until every already-admitted
old-epoch execution has reached its declared durability event or terminally
aborted, and only then activates the new epoch. Rejecting future old-epoch
admissions is insufficient while any old-epoch execution remains active.

Retention does not establish applicability. Historical executable versions
remain loadable; protected current policy controls revocation and
disallowance. Admission requires
\begin{equation}
  \bigwedge_{r_f\in\mathcal R_X}
\mathsf{ExecAllowed}_{\Pi_{\mathrm{commit}}}(r_f).
\label{eq:registry_applicability}
\end{equation}
This executable-applicability decision is a protected policy dependency in
$J_P$, not a conclusion inferred from the retained content.
There is no unguarded mutable ``currently valid definition'' bit. Mutable maps
used to discover operation, policy, and recertifier heads are protected by a
coarse shared discovery guard
\[
g^{\mathrm{disc}}_{\langle tenant,op\rangle},
\]
whose key is derivable from sealed tenant and operation identifiers before any
mutable registry read. Every writer of a covered map takes the matching
exclusive guard. The protected maps include the admissible normalizer epoch
for each affected coordination domain; a mapping-changing epoch transition
holds the corresponding discovery/normalizer-epoch guard exclusively through
the stop-admission, drain-or-terminal-abort, and activation sequence above.

\subsection{Typed Dependencies and Applicability}

\begin{definition}[Dependency token]
\label{def:dependency_token}
A dependency token is
\begin{equation}
\delta=\langle k,q,v,r_{\varphi},\iota,a,r_{\chi}\rangle ,
\label{eq:dependency_token}
\end{equation}
where $k$ is a dependency kind; $q$ an object, predicate, or scope key; $v$ the
canonical value exposed to reasoning; $r_{\varphi}$ an optional versioned
per-token recertifier reference; $\iota$ an issuer; $a$ an assertion or digest;
and $r_{\chi}=\langle\mathit{id}_{\chi},\nu_{\chi},h_{\chi}\rangle$ the versioned
validity-semantics definition. All references and digests are bound by the
envelope.
\end{definition}

External artifacts require two different dependency forms:
\begin{align}
\delta_{\mathrm{content}}={}&
 \langle artifact\_id,content\_digest, \nonumber\\
 &\quad issuer,immutability\_proof\rangle ,
\label{eq:content_token}\\
\delta_{\mathrm{selection}}={}&
 \langle issuer,query\_or\_subject, \nonumber\\
 &\quad selected\_result,head\_or\_epoch, \nonumber\\
 &\quad applicability\_predicate,r_{\chi}\rangle .
\label{eq:selection_token}
\end{align}
The content token proves only that the observed bytes remain available and
unchanged in the declared artifact store. It does not prove that the artifact
is current, authoritative, applicable, unrevoked, or unsuperseded. If a
mediated interface returns ``the current $X$,'' capture includes both the
returned content token and the selection token that records why that artifact
was current. A historical artifact needs no currentness token only when its
historical existence, rather than present applicability, is the actual premise.

The mediator may expose a stable trusted abstraction as $v$, but then the
reasoner receives only that abstraction. Raw content exposed alongside it
receives its own content token. $J_D$ covers row, object, predicate, and
aggregate reads; $J_E$ covers content and selection dependencies plus belief
premises; $J_P$ binds $\Pi_{\mathrm{obs}}$ and its registry selection; and
$J_{A\mathrm{obs}}$ contains only authority premises, approval observations,
capability observations, and other authority values actually exposed during
derivation. It contains no grant request or commit-time grant. Such
observations remain ordinary derivation dependencies; the separately derived
$Q_X$ determines which mandatory authority and external-validity obligations
the concrete operation must discharge.

For current authoritative state $S_q(s)$, faithfulness is typed:
\begin{equation}
\begin{aligned}
&\mathsf{Faithful}(\delta_{\mathrm{content}},S(s))\\[-1mm]
&\quad\equiv
 \mathsf{Digest}(\mathsf{Content}(artifact\_id))=content\_digest\\[-1mm]
&\qquad\land\mathsf{VerifyImmutable}(immutability\_proof).
\end{aligned}
\label{eq:faithful_content}
\end{equation}
\begin{equation}
\begin{aligned}
&\mathsf{Faithful}(\delta_{\mathrm{selection}},S(s))\\[-1mm]
&\quad\equiv\mathsf{Head}_{issuer}(query,s)=head\_or\_epoch\\[-1mm]
&\qquad\land\mathsf{Select}_{issuer}(query,s)=selected\_result\\[-1mm]
&\qquad\land\mathsf{Applicable}_{r_{\chi}}
        (selected\_result,S(s)).
\end{aligned}
\label{eq:faithful_selection}
\end{equation}
For other token types,
\begin{equation}
\mathsf{Faithful}(\delta,S_q(s))
 \equiv\mathsf{Obs}_{\delta}(S_q(s))=v_{\delta}
 \land\mathsf{Valid}_{r_{\chi}}(\delta,S_q(s)).
\label{eq:faithful}
\end{equation}
TCT-S validates artifact integrity and current selection whenever derivation
used both premises.

A signed point observation states only what held at one issuer-clock instant.
A revocable attestation requires a trustworthy revocation mechanism. A future
expiry field alone supplies no commit fence. Any interval, grant, applicability
predicate, or per-token optimization is interpreted only through its versioned
$r_{\chi}$ or $r_{\varphi}$ definition.

\subsection{Closed Inputs and Trusted Capture}

\begin{definition}[Closed-input mediated capture]
Every operational input made available to the reasoner---database results,
retrievals, tool results, files, intent fields, epistemic state, cached memory,
belief premises, human approvals, policies, and credentials---enters through a
trusted mediated channel and appears only through $I_J$, $K_J$, or
$\mathsf{View}(J_{\mathrm{obs}})$. Every exposed value is represented by a
token in $J_{\mathrm{obs}}$. A
transformation or summary carries transitive provenance to its sources. The
mediator records a conservative superset $J_{\mathrm{obs}}$; it does not claim
to identify every fact that causally influenced a stochastic model.
\end{definition}

Parametric model knowledge is outside the contract unless materialized through
a mediated channel. Prompt injection cannot delete mediator records, but bypass
channels and compromised mediators violate the assumption. This resembles the
observable-read boundary in S-Bus~\cite{sbus2026} and the protected-boundary
assumption in CommitGuard~\cite{commit_authorization2026}.

\subsection{Canonical Envelope and Seal}

Let $\mathsf{Canon}$ be an injective canonical serialization with fixed field
ordering, type tags, normalized encodings, and deterministic token ordering.
Every canonical object also carries a protocol-versioned, type-specific domain
tag, so an envelope, effect, token, plan, witness, grant, receipt, or release
proof cannot be substituted for another object class. TCT assumes
collision and second-preimage resistance for $\mathsf{Hash}$, unforgeability
for the signature and MAC schemes, and authenticated mediator, issuer, and
terminal-proof keys bound to their declared roles and tenants.
Let $r_{\Omega}=\langle op,\nu_{\Omega},h_{\Omega}\rangle$ identify the immutable
operation specification and let $r_{\Pi,o}$ identify $\Pi_{\mathrm{obs}}$. The
sealed payload is
\begin{align}
\bar X={}&\langle id,tenant,h_c,I_J,K_J,op,p,\mu,h_{\mu},m_X,\Delta B,
  J_{\mathrm{obs}}, \nonumber\\
& Q_X,h_{\mathsf{Atoms},X},r_{\Omega},\mathcal R_X,r_{\Pi,o},\Pi^{+},
  \mathcal D_X,\tau\rangle , \nonumber\\
h_X={}&\mathsf{Hash}(\mathsf{Canon}(\bar X)),
\label{eq:envelope_digest}
\end{align}
where $h_{\mu}=\mathsf{Hash}(\mathsf{Canon}(\mu))$, $\mathcal D_X$ is the
sealed target commit-domain identity and durability-semantics reference, and
$\tau=[t_{\min},t_{\max}]$ is a committer-clock admission window.
$\mathsf{Canon}(\bar X)$ excludes $h_X$ and the seal.

\begin{definition}[Sealed cognitive transaction envelope]
A sealed envelope is $X=\langle\bar X,h_X,\zeta\rangle$, where $\zeta$ is either
(i) a signature or MAC by the trusted mediator over $(id,h_X)$ or (ii) a
reference to an immutable trusted registration record containing $(id,h_X)$.
The mediator allocates $id$ once and retains an immutable
$\mathsf{IdBind}(id,h_X)$ record even if an effect attempt aborts; it never
binds that $id$ to another digest. Admission recomputes $h_X$ and $h_{\mu}$ and
verifies the identity binding, $\zeta$, every executable reference,
predicate/semantics version, and digest before use.
$J_{A\mathrm{obs}}$ contains only sealed observations, while the trusted
mandatory plan $Q_X$ is separately sealed. The commit-time
grant set
\[
  \Gamma_X=
  \langle j\mapsto\gamma_{X,j}\rangle_{j\in\mathsf{ord}(Q_X^B)}
\]
is a canonical map keyed by sealed plan ordinal and is obtained only after
$h_X$ exists;
$\Gamma_X\cap(J_{\mathrm{obs}}\cup Q_X)=\varnothing$ and is excluded from
$\bar X$, $\mathsf{Canon}(\bar X)$, and $h_X$. Thus no grant signature that
binds $h_X$ can be an input to $h_X$. Admission checks exact coverage of the
sealed plan, not a relation to capture-time authority observations.
\end{definition}

The seal proves payload integrity and mediator provenance, not external truth
or current applicability. The commit receipt binds $r_{\Pi,o}$,
$r_{\Pi,c}$, $r_{\Phi}$, $Q_X$, the successful $Q_{\mathrm{req}}=Q_X$ check,
the $\mathsf{DepCovered}$ discharge map, $\mathcal D_X$, the exact post-seal
$\Gamma_X$ (or immutable content-addressed references to it), and the
resulting validation verdict.

\section{Cognitive Serializability}
\label{sec:cognitive-serializability}

\subsection{Joint History}

A joint history $H$ is a partially ordered set of typed events with
per-participant program order and observed communication order. Its alphabet
contains database operations; artifact observations and selection changes;
policy-epoch changes; authority grants, revocations, and expirations; seal and
verification events; guard, grant, prepare, validation, commit, abort, receipt,
and epistemic-state events. The database projection $H_D$ retains operations of
committed envelopes only.

\subsection{Serialization and Durability Events}

The sealed object $\mathcal D_X$ identifies the target gate, database or
atomic-commit domain, and the versioned semantics of its durability condition.
For committed transaction $T_i$, $d_i$ is the runtime event at which that
condition is actually reached in $\mathcal D_{X_i}$:
\[
  d_i=\mathsf{ReachDurability}(T_i,\mathcal D_{X_i}).
\]
A grant can therefore sign $\mathcal D_{X_i}$ before commit, but not the future
event $d_i$.

\begin{definition}[Cognitive serialization event]
\label{def:serialization_event}
For a committed envelope $X_i$, $s_i$ is the logical event after the complete
canonical local guard set has been acquired, all protected local observations
have been made under those guards, and every required post-seal grant in
$\Gamma_{X_i}$ and other external validity mechanism has accepted. If $a_{ij}$
is the acceptance event for external dependency $j$, then
$a_{ij}\preceq s_i\preceq d_i$, and every
grant-protected interval contains $s_i$ and extends through $d_i$.
The single growing phase over local guards and incompatible external
reservations ends at lock point $\lambda_i\preceq s_i$.
No protected state is read or validated before the final local guard. Local
guards and external reservations remain effective until the database reaches
the declared runtime durability event $d_i$.
\end{definition}

$s_i$ is an event, not a timestamp. Let $C_G$ be the committer clock and
$t_i=C_G(s_i)$. Durability semantics named by $\mathcal D_X$ may require
primary-WAL flush, configured synchronous-standby acknowledgments, or a
specified atomic-commit condition. A successful SQL \texttt{COMMIT} is
interpreted only relative to that sealed target; $d_i$ is the corresponding
event in the produced history.

\subsection{Commit-Spanning Validity Mechanisms}

Four mechanisms can preserve dependency validity through commit:
\begin{description}
  \item[A. Co-transactional dependency.] The authoritative validity state is
  read and updated in the same ACID transaction and protected by the same guard
  discipline.
  \item[B. Envelope- and witness-bound grant.] Issuer $\iota_j$ returns
  \begin{multline}
  \Gamma_{X_i}=
  \langle j\mapsto\gamma_{X_i,j}\rangle_{j\in\mathsf{ord}(Q_{X_i}^B)},
  \qquad
  \gamma_{X_i,j}=\mathsf{Sign}_{\iota_j}\langle\\
  id_i,h_{X_i},h_{\mu_i},h_{Q_{X_i}},\mathcal D_{X_i},j,m_{X_i},\\
  \mathit{scope}_{ij},m_{ij},h_{\mathsf{Cover}_{ij}},\\
  h_{w,ij},r_{w,ij},r_{\chi,ij},
  \mathit{decision}_{ij},n_{ij}\rangle ,
  \label{eq:typed_grant}
  \end{multline}
  where $h_{Q_{X_i}}$ digests sealed $Q_{X_i}$,
  $h_{\mathsf{Cover}_{ij}}$ digests the exact covered-token vector, $m_{ij}$ is
  the S/X mode, $r_{w,ij}$ is the witness constructor, $r_{\chi,ij}$ identifies
  validity semantics, and $n_{ij}$ is a single-use nonce. The issuer receives
  the matching sealed plan item $q_{ij}\in Q_{X_i}^{B}$, where
  $Q_{X_i}^{B}=\{q\in Q_{X_i}:M(q)=B\}$, and creates $\gamma_{X_i,j}$ only
  after verifying $h_{X_i}$. Acceptance atomically places the grant in
  \textsc{Reserved} and promises to preserve its exact witness until runtime
  event $d_i$ occurs in the signed target $\mathcal D_{X_i}$. Verification
  requires exact equality of every plan, envelope, effect, and target-domain
  field; a grant for another envelope or commit domain is inapplicable even
  when $h_{\mu}$ is equal. This is an envelope-specific reservation, analogous
  in purpose to transactional escrow~\cite{oneil1986escrow}.
  \item[C. Atomic-commit participant.] The external service participates in a
  prepare/commit protocol whose decision includes reaching $d_i$ in
  $\mathcal D_{X_i}$~\cite{gray1992transaction}. Atomic participation supplies atomicity,
  not isolation: each participant also provides serializable isolation or
  compatible reservation ordering for the witnessed state, and conflicting
  orders across participants obey the common-order contract below.
  \item[D. Hard-real-time certificate.] The issuer guarantees validity over
  $[\ell^{-},\ell^{+}]$, clocks obey a proved skew bound, and the commit path has
  a genuinely enforced worst-case duration $b$.
\end{description}

\paragraph{Mechanism selection.}
The trusted registry selects the mechanism for each dependency; the caller
cannot choose a weaker discharge. Mechanism A applies when authoritative state
can share the database transaction and writer discipline. Mechanism B applies
when an issuer can reserve an exact witnessed value through the target
durability event but cannot join the database commit. Mechanism C applies when
the external service can participate in atomic commit and also provide the
required isolation or reservation order. Mechanism D is appropriate only when
the deployment can prove its clock and worst-case completion assumptions.
Failure to satisfy any mechanism's full contract causes rejection rather than
fallback to a weaker check.

For mechanism D only, if each trusted clock differs from ideal time by at most
$\epsilon$, the conservative sufficient check is
\begin{equation}
 \ell^{-}+2\epsilon\leq t_i
 \quad\land\quad t_i+b+2\epsilon\leq\ell^{+}.
\label{eq:lease_margin}
\end{equation}
An expiry timestamp, PostgreSQL \texttt{statement\_timeout}, or an ordinary
operational latency target does not prove a hard bound through durability.
The PostgreSQL realization in Section~\ref{sec:legacy} uses A for local state
and B for external selection/applicability and authority; it does not rely on D.

For every external plan item $q_{ij}$, let
$C_{ij}=\mathsf{Cover}(q_{ij})$ and let $S_{C_{ij}}(u)$ be the canonical
authoritative vector for exactly those covered dependency tokens at event $u$.
At issuer or participant acceptance $a_{ij}$, the canonical witness payload
and its commitment are
\begin{equation}
\begin{aligned}
w_{ij}&=\mathsf{Witness}_{r_{w,ij}}(S_{C_{ij}}(a_{ij})),\\
h_{w,ij}&=\mathsf{Hash}(\mathsf{Canon}(w_{ij})).
\end{aligned}
\label{eq:grant_witness}
\end{equation}
For B, the issuer returns $w_{ij}$ with the signed grant, or an immutable
authenticated content-addressed reference whose resolved bytes hash to
$h_{w,ij}$; the signature continues to bind $h_{w,ij}$. Mechanisms C and D
provide the same authenticated value or reference whenever their values enter
$S_{J_{\mathrm{obs}}}(s_i)$. The payload contains, or reconstructs, the exact
current dependency-vector components consumed by $\mathsf{Faithful}$,
$\mathsf{ResolvePlan}$ and $Q_{\mathrm{req}}$, and $\Phi$. An opaque digest,
Boolean decision, or authorization-only summary is insufficient unless those
predicates are formally defined over commitments rather than raw values. For
TCT-S it must equal the value or registered predicate represented during
derivation by the tokens in $C_{ij}$. For TCT-C it must be the exact current
vector component supplied to $\Phi$, not a separate authorization-only
summary.

Let $a_{ij}$ be issuer acceptance. The issuer contract is
\begin{equation}
\begin{aligned}
&\text{for all }u\in[a_{ij},d_i]:\\[-1mm]
&\quad\mathsf{Hold}(\gamma_{X_i,j},u)\equiv
 \chi_{r_{\chi,ij}}\bigl(S_{C_{ij}}(u),w_{ij},r_{w,ij}\bigr)\\
&\qquad{}\land
 \mathsf{Witness}_{r_{w,ij}}\bigl(S_{C_{ij}}(u)\bigr)=w_{ij}.
\end{aligned}
\label{eq:grant_holding}
\end{equation}
For an exact-value token, $\chi$ requires equality to the witnessed value; for
a predicate token, it requires the registered predicate represented by
$h_{w,ij}$. A promise that the effect remains authorized does not establish
that a reported raw value remains current. An unconsumed grant may be released
only after authoritative proof, bound to
$(id_i,h_{X_i},h_{Q_{X_i}},j,n_{ij})$, that the corresponding effect
transaction is terminally unable to commit. Timeout,
missing response, and temporary absence of a receipt are insufficient; a
reservation leak after an uncertain outcome is a liveness failure, not
permission to weaken safety.

The issuer treats
$\langle id_i,h_{X_i},h_{Q_{X_i}},j\rangle$ (equivalently, the canonical
plan-item digest) as the idempotency key. Scope is a signed field but is not the
identity key because two plan items may intentionally share a scope. A
same-envelope retry recovers the existing reservation or, after success, the
receipt recorded for that reservation. It may issue a replacement only after
verifying authoritative abort-and-release proof for the previous grant; a
timeout or absent receipt cannot authorize replacement.

Let
$\kappa_{ij}=\langle id_i,h_{X_i},h_{\mu_i},h_{Q_{X_i}},j,n_{ij}\rangle$.
The issuer-side terminal lifecycle is
\[
\begin{aligned}
\textsc{Reserved}
&\xrightarrow{\;\mathsf{Finalize}(\rho_i,\kappa_{ij},\mathcal D_{X_i})\;}
   \textsc{Consumed},\\
\textsc{Reserved}
 &\xrightarrow{\;\mathsf{TerminalAbortProof}\;}
   \textsc{Released}.
\end{aligned}
\]
$\mathsf{Finalize}$ requires a durable receipt or equivalent commit proof for
the signed tuple. It is delivered by an idempotent post-commit outbox or an
equivalent receipt-finalization path; duplicate delivery leaves the grant
\textsc{Consumed}. No timeout, lost response, or missing receipt enables the
second transition. Indefinite retention in \textsc{Reserved} is safe but is a
liveness failure.

Multiple external reservations and participant conflicts require a common
order. Normalize every B/C scope into canonical atomic acquisition identities
$a$. Define the resource-invariant key
\[
\begin{aligned}
k_{\mathrm{atom}}(a)=\mathsf{Canon}\langle&
\mathit{coordinationDomain}(a),\\[-1mm]
&\mathit{atomicLockId}(a)\rangle .
\end{aligned}
\]
The same physical lock has exactly the same coordination domain, atomic lock
identifier, and key regardless of parent scope, envelope, plan ordinal,
mechanism label, or descriptive spelling. Any two scopes that can conflict
normalize to at least one common identity. Within an envelope, duplicate
identities are merged, $X$ dominates $S$, and the merged item retains all
covered plan ordinals. The gate sorts the resulting atomic list by
$k_{\mathrm{atom}}$, performs each actual B or C atomic reservation in that
order, and holds incompatible reservations through $d_i$. If incompatible
acquisitions cannot share a coordination domain, a common coordinator must
establish their order. A registered independence/commutativity proof may
remove a precedence edge, but does not waive a required reservation, grant,
participant decision, certificate, or authenticated witness. Plan ordinal
remains the grant identity and idempotency coordinate, never the concurrency
key. Atomic normalization supports deadlock freedom; it does not change the
lock-point safety argument.
If an endpoint cannot expose composable atomic reservations, its whole scope
is one canonical atomic acquisition identity; if that identity cannot share a
coordination domain with an incompatible acquisition, the common-coordinator
alternative is required. Plan-level grant and participant-decision assembly
after the sorted atomic loop performs no additional resource acquisition.

\subsection{Strict and Compatible Properties}

\begin{definition}[Strict Cognitive Serializability]
\label{def:cognitive_serializability}
A committed joint history $H$ is \emph{strictly cognitively serializable} if
there exist a total order $\prec$ over committed envelopes and an event $s_i$
for each $X_i$ such that:
\begin{enumerate}
  \item $H_D$ is equivalent to serial execution of the committed registered
  effects in order $\prec$;
  \item every database-conflict, guard, required real-time,
  $\Gamma_{X_i}$-grant/prepare, and validity-boundary edge is consistent with
  $\prec$ and $s_i$;
  \item every $\delta\in J_{\mathrm{obs},i}$ is faithful at $s_i$, including
  both content and selection tokens when the interface returned a current
  artifact;
  \item
  $\mathsf{Canon}(\Pi_{\mathrm{commit},i})
   =\mathsf{Canon}(\Pi_{\mathrm{obs},i})$; and
  \item every $q_{ij}\in Q_{X_i}$ is discharged by its sealed mechanism through
  $d_i$, with all protected intervals containing $s_i$; for
  $q_{ij}\in Q_{X_i}^{B}$, Equation~\ref{eq:grant_witness} holds and the
  matching $\gamma_{X_i,j}\in\Gamma_{X_i}$ supplies the witness. No grant is a
  member of $J_{\mathrm{obs},i}$, and both
  $\mathsf{DepCovered}(J_{\mathrm{obs},i},Q_{X_i})$ and
  $\mathsf{PlanCurrent}(Q_{X_i},s_i)$ hold.
\end{enumerate}
\end{definition}

\paragraph{Terminology.}
The adjective \emph{strict} denotes derivation faithfulness plus
commit-spanning validity; it does not assert classical strict serializability
or external consistency. TCT orders the real-time edges required by the sealed
contract, but it imposes no global wall-clock order on unrelated transactions
unless the underlying deployment supplies that stronger property.

Immutability discharges only the content-token clause. Currentness,
authority, and applicability require their own faithful selection token and
mechanism A--D. If raw input or current selection changes and only the effect is
recertified, the history is not derivation-faithful.

Let
\[
S_{J_{\mathrm{obs}}}(s)=
\langle S_{\delta}(s)\rangle_{\delta\in J_{\mathrm{obs}}}
\]
be the complete current dependency vector. The operation and commit-time policy
references resolve a versioned envelope-level predicate:
\begin{align}
&\mathsf{Compatible}_X
 (X,S_{J_{\mathrm{obs}}}(s),\mu,\Pi_{\mathrm{commit}},t) \nonumber\\
&\quad\equiv
\Phi_{r_{\Phi}}\bigl(h_X,\mathsf{Canon}(J_{\mathrm{obs}}),
 S_{J_{\mathrm{obs}}}(s), \nonumber\\[-1mm]
&\hspace{37mm}\mu,\Pi_{\mathrm{commit}},t\bigr),
\label{eq:joint_compatible}
\end{align}
where trusted resolution selects $r_{\Phi}$ from
$(r_{\Omega},r_{\Pi,c})$; equivalently, this is the registered family
$\Phi_{r_{\Omega},r_{\Pi,c}}$ instantiated at that version. Passing $h_X$
makes the predicate
envelope-relative, rather than a conjunction of token-local approvals.

\begin{definition}[Effect-Compatible Cognitive Admission]
\label{def:effect_compatible_admission}
A committed envelope is \emph{effect-compatibly admitted} if it has a database
serial position and event $s_i$ satisfying items 1, 2, and 5 above,
$\Pi_{\mathrm{commit}}$ is resolved under the protected current policy head,
every mandatory $q_{ij}\in Q_{X_i}$ is discharged, and the complete witness
vector $S_{J_{\mathrm{obs},i}}(s_i)$ existed simultaneously under overlapping
protection extending through $d_i$. Every mutable vector
component supplied to $\Phi$ is either co-transactionally guarded or covered by
a mechanism-B, C, or D interval containing $s_i$ and extending through $d_i$;
self-stable components use their verified immutable rule. The envelope also
satisfies $\mathsf{DepCovered}(J_{\mathrm{obs},i},Q_{X_i})$ and
$\mathsf{PlanCurrent}(Q_{X_i},s_i)$, and
\begin{equation}
\mathsf{Compatible}_X
(X_i,S_{J_{\mathrm{obs},i}}(s_i),\mu_i,
 \Pi_{\mathrm{commit},i},t_i).
\label{eq:compatible_admission}
\end{equation}
\end{definition}

A compatible decision cannot combine values current at different times. Each
mutable input to $\Phi$ is the value protected at $s_i$ through $d_i$; for B,
it is the witness in Equation~\ref{eq:grant_holding}. Effect authorization
alone does not prove that this vector existed.

$\Phi$ evaluates the complete current vector and all cross-dependency
invariants. Per-token predicates may be cached or evaluated as an optimization
only when the trusted registry binds a versioned decomposition certificate
proving
\[
\bigwedge_{\delta\in J_{\mathrm{obs}}}\varphi_{\delta}
\;\Longrightarrow\;\mathsf{Compatible}_X.
\]
Without that sufficiency result, independent conjunction is not an admission
rule. For example, an order of 60 units may remain valid if exposure alone rises
from 20 to 45 under the observed cap of 100, and also remain valid if the cap
alone falls from 100 to 80 under the observed exposure of 20. Together the
current exposure and cap require $60+45\leq80$, which is false. Only joint
recertification observes the invalid combination.

Effect-compatible admission retains $\Pi_{\mathrm{obs}}$ as justification
history but validates the sealed effect under $\Pi_{\mathrm{commit}}$; it does
not serialize the original derivation.

\subsection{TCT Profiles and Composition}

\begin{definition}[TCT contract]
\emph{TCT-S Safety} combines strict Cognitive Serializability, exact-effect
authorization by
\[
\mathsf{AuthorizedExact}
 (Q_X,\Gamma_X,id,h_X,h_{\mu},\mathcal D_X,
  S_{J_{\mathrm{obs}}}(s),d),
\]
justification--commit correspondence in which the receipt binds both sealed
$X$ and post-seal $\Gamma_X$, and no-premature belief activation.
\emph{TCT-C Safety} substitutes joint Effect-Compatible Cognitive Admission
and must be labeled as such. Either profile adds conditional eventual
reconciliation for \emph{TCT Correctness}.
\end{definition}

\begin{proposition}[Constraint-relative non-implication]
\label{prop:sea_separation}
For an ATP constraint set $\mathcal C$ that omits a commit-relevant dependency,
cross-dependency invariant, or commit-spanning mechanism,
$\mathsf{SEA}(\mathcal C)$ implies neither TCT-S nor TCT-C Safety.
\end{proposition}

\begin{proof}
Choose two executions identical on every predicate in $\mathcal C$ but with the
omitted dependency or joint invariant valid in one and invalid in the other.
Admission may accept both while satisfying SEA relative to $\mathcal C$; the
second violates item 3 or 5 of
Definition~\ref{def:cognitive_serializability} under TCT-S, or
Definition~\ref{def:effect_compatible_admission} under TCT-C. Mnemosyne states
that its guarantees are relative to $\mathcal C$~\cite{mnemosyne2026}; this is
a contract-coverage result.
\end{proof}

\begin{proposition}[ATP embedding]
\label{prop:atp_embedding}
An ATP runtime can implement either TCT profile when $\mathcal C$ contains the
complete typed dependency vector, trusted sealed $Q_X$, joint recertifier,
policy-resolution rules, selected validity mechanisms, post-seal $\Gamma_X$,
the $\mathsf{DepCovered}$ and $\mathsf{PlanCurrent}$ checks, and atomic
effect/receipt contract.
\end{proposition}

\begin{proof}
Proposal non-authority and gate closure prevent bypass; complete constraints
perform the selected TCT admission rule. Mechanism A--D protects the decision
and, where the rule consumes raw state, its exact simultaneous witnesses
through runtime event $d_i$ in sealed $\mathcal D_{X_i}$. The common-order
contract orders incompatible external
reservations, and atomic effect/receipt persistence supplies correspondence.
\end{proof}

\begin{corollary}[Strict admission-core reduction]
\label{prop:local_reduction}
If all dependencies and policy are protected local database state,
$\Delta B=\varnothing$, $\Omega_{\mu}$ completely declares the footprint, and
authorization is absent, local, or vacuously true, then the validation core of
$\mathsf{Admit}_S$ reduces to conventional exact read-set validation followed
by atomic application of $W_D(\mu)$.
\end{corollary}

\begin{proof}
Content selection and external fences disappear. After the complete guard set
is held, faithfulness is exact local version validation over the state named by
$\mathsf{LocalGuards}_A(J_{\mathrm{obs}},Q_X)$ and $R_D(\mu)$, followed by the
registered write set and mandatory invariants. This is the classical validation
shape~\cite{kung1981occ}. The
reduction concerns admission validation, not the complete TCT-S bundle:
canonical sealing, idempotency, and effect--receipt correspondence remain;
$\Delta B=\varnothing$ removes only belief reconciliation.
\end{proof}

\section{The TCT Commit Protocol}
\label{sec:protocol}

\subsection{Lifecycle}

\begin{enumerate}
  \item \textbf{Observe.} Trusted mediation produces $I_J$, $K_J$, and
  $J_{\mathrm{obs}}$,
  captures
  $r_{\Pi,o}$, loads the corresponding $\Pi_{\mathrm{obs}}$, and exposes that
  policy during derivation. The reasoner receives only these mediated
  projections. A ``current artifact'' response yields both content and
  selection tokens.
  \item \textbf{Reason and resolve.} The reasoner proposes
  $(op,p,\Delta B,\Pi^{+})$ while receiving $\Pi_{\mathrm{obs}}$. Trusted
  infrastructure computes
  $\mathsf{ResolveSeal}(op,p,c,J_{\mathrm{obs}},\Pi^{+})
   =\langle\mu,\Omega_{\mu},Q_X,m_X\rangle$
  and makes $\Delta B$ tentative. The trusted
  registry supplies every mandatory plan item; optional requests can only
  strengthen that plan, and resolution rejects unless
  $\mathsf{DepCovered}(J_{\mathrm{obs}},Q_X)$. Resolution does not relabel a
  post-reasoning policy as observed.
  \item \textbf{Seal.} The mediator seals canonical $\bar X$, including the
  content/selection dependencies, $m_X$, $r_{\Omega}$, the complete executable
  reference set $\mathcal R_X$ (including $r_{\mathrm{norm}}$),
  $h_{\mathsf{Atoms},X}$, $r_{\Pi,o}$, $\Pi^{+}$, the exact $Q_X$,
  $\mathcal D_X$, and the exact effect digest. Its permanent
  $\mathsf{IdBind}(id,h_X)$ record prevents
  later reuse of $id$ for different sealed bytes, including after abort.
  $\Gamma_X$ does not yet exist and is not part of the sealed payload.
  \item \textbf{Guard and identify.} After prechecking immutable sealed bytes, a
  short \texttt{READ COMMITTED} transaction acquires the complete canonical
  non-ID guard set as key--mode pairs. The set includes the coarse registry
  discovery guard derivable from sealed tenant and operation identifiers. It
  then attempts the unique $(id,h_X)$ insertion, which realizes the final
  $g_{id}$ guard, before any protected read.
  \item \textbf{Resolve current state.} Under the held guards, the gate reads
  local state, resolves $(\Pi_{\mathrm{commit}},r_{\Pi,c},r_{\Phi})$, verifies
  applicability of $\mathcal R_X$, and executes $Q_X$. It normalizes B/C
  scopes, merges duplicate identities at the strongest mode, and reserves
  them in $k_{\mathrm{atom}}$ order (or through a common coordinator).
  $\Gamma_X$ covers exactly the B items. Event $s_i$ follows all acceptances
  and witnesses. Recomputed current plan and profile must equal sealed $Q_X$
  and $m_X$; otherwise the gate aborts and requires resealing.
  \item \textbf{Validate and commit.} TCT-S checks
  $\mathsf{Faithful}(\delta,S_{\delta}(s))$ for every
  $\delta\in J_{\mathrm{obs}}$, including database, evidence, belief, policy,
  and observed-authority tokens under every selected mechanism. TCT-C applies
  the registered joint recertifier to the complete simultaneously protected
  current vector and $\Pi_{\mathrm{commit}}$.
  All validation reads share one protected pre-effect state $\sigma_s^{-}$;
  only successful validation may create
  $\sigma_s^{+}=\mu(\sigma_s^{-})$. The typed effect, canonical bytes,
  dependencies, $Q_X$,
  observed and commit policy references, executable references, verdict,
  $\Gamma_X$ (or immutable references to its grants), receipt, and outbox rows
  reach runtime event $d_i$ in $\mathcal D_X$ atomically.
  \item \textbf{Finalize grants.} An idempotent post-commit outbox delivers
  $\rho_i$ or equivalent durability proof for each
  $(id,h_X,h_{\mu},h_{Q_X},j,n_j,\mathcal D_X)$. The issuer moves the matching
  reservation from \textsc{Reserved} to \textsc{Consumed}; duplicate delivery
  is a no-op.
  \item \textbf{Reconcile.} Delivery of $\rho_i$ moves $\Delta B_i$ to
  \textsc{Committed Pending Reconciliation}. ESR may activate, reject, or merge
  it under its own rules~\cite{esr2026}.
\end{enumerate}

\subsection{Admission Contract}

Let $S_{J_{\mathrm{obs}}}(s)$ denote the complete current dependency vector
observed after the final guard. Every committed execution satisfies
\begin{multline}
\mathsf{Base}(X,\Gamma_X,s,t,d)\equiv
 \mathsf{SealOK}(X)\land\mathsf{EffectDigestOK}(X)\\
{}\land\mathsf{RegistryOK}(X,r_{\Pi,c},r_{\Phi})\\
{}\land\mathsf{DepCovered}(J_{\mathrm{obs}},Q_X)\\
{}\land\mathsf{PlanCurrent}(Q_X,s)\\
{}\land\mathsf{ProfileCurrent}(m_X,s)\\
{}\land\bigwedge_{r_f\in\mathcal R_X}
 \mathsf{ExecAllowed}_{\Pi_{\mathrm{commit}}}(r_f)\\
{}\land t_{\min}\leq t\leq t_{\max}\\
{}\land\mathsf{MechanismThroughD}
 (J_{\mathrm{obs}},Q_X,\Gamma_X,id,h_X,h_{\mu},\\
\mathcal D_X,S_{J_{\mathrm{obs}}}(s),d)\\
{}\land\mathsf{AuthorizedExact}
 (Q_X,\Gamma_X,id,h_X,h_{\mu},\mathcal D_X,\\
S_{J_{\mathrm{obs}}}(s),d)\\
{}\land\mathsf{Pre}_{\mu}(\sigma_s^{-})
{}\land\mathsf{Preserves}_{\Pi_{\mathrm{eff}}}(\mu,\sigma_s^{-}).
\label{eq:base_admit}
\end{multline}
$\mathsf{RegistryOK}$ verifies $r_{\Omega}$, $\mathcal R_X$, and both policy
and recertifier references. $\mathsf{PlanCurrent}$ recomputes
$Q_{\mathrm{req}}(s)$ under protected current state and requires canonical
equality with sealed $Q_X$; it cannot extend the plan after sealing.
$\mathsf{MechanismThroughD}$ requires complete execution of every
$q_j\in Q_X$ by its sealed mechanism $M_j$, verifies the protected witness
constructed from exactly $\mathsf{Cover}(q_j)$, makes its authenticated
canonical payload operationally available, and verifies that every
required holding interval contains $s$ and extends through the actual
durability event $d=\mathsf{ReachDurability}(T,\mathcal D_X)$. For mechanism B
it rejects an omitted or unexpected grant and verifies the exact issuer, scope,
mode, decision, semantics, $m_X$, witness payload and digest, $id$, $h_X$, $h_{\mu}$, and
$\mathcal D_X$ bindings before accepting that holding interval.

$\mathsf{AuthorizedExact}$ requires
\[
\mathsf{dom}(\Gamma_X)=\mathsf{ord}(Q_X^B)
\]
under plan ordinal: every mechanism-B obligation has one grant and no grant is
unexpected. For each matched pair it verifies exact issuer, scope, S/X mode,
covered-token vector, witness-construction reference, validity-semantics
reference, required decision, $m_X$, $id$, $h_X$, $h_{\mu}$, $h_{Q_X}$,
$\mathcal D_X$, and single-use nonce; it recomputes
Equation~\ref{eq:grant_witness}, authenticates $w_j$, requires
$\mathsf{Witness}_{r_{w,j}}(S_{\mathsf{Cover}(q_j)}(s))=w_j$, and verifies
holding through $d$. Plan items using A, C, or D admit no substitute grant and
are discharged only by $\mathsf{MechanismThroughD}$. Thus exact authorization
is relative to trusted $Q_X$, never to $J_{A\mathrm{obs}}$. In TCT-S, the
recomputed witness must equal the value or predicate represented during
derivation; in TCT-C, it must be the identical current vector component passed
to $\Phi$. $\Gamma_X$ is the canonical plan-ordinal map defined in
Section~\ref{sec:system-model}. Write $\Gamma=\Gamma_X$ below. The two profiles
are
\begin{multline}
\mathsf{Admit}_{S}(X,\Gamma,s,t,d)\equiv\\
 \mathsf{Base}(X,\Gamma,s,t,d)\\
{}\land
 \mathsf{Canon}(\Pi_{\mathrm{commit}})
 =\mathsf{Canon}(\Pi_{\mathrm{obs}})\\
{}\land\bigwedge_{\delta\in J_{\mathrm{obs}}}
 \mathsf{Faithful}(\delta,S_{\delta}(s)),
\label{eq:strict_admit}
\end{multline}
\begin{multline}
\mathsf{Admit}_{C}(X,\Gamma,s,t,d)\equiv\\
 \mathsf{Base}(X,\Gamma,s,t,d)\\
{}\land\mathsf{Compatible}_X
 \bigl(X,S_{J_{\mathrm{obs}}}(s),
 \mu,\Pi_{\mathrm{commit}},t\bigr).
\label{eq:compatible_admit}
\end{multline}
Each $\gamma_{X,j}$ binds its $Q_X$ item, covered vector, witness function,
envelope, effect, and durability domain as specified in
Equation~\ref{eq:typed_grant}. It cannot authorize another envelope, witness
subset, or domain, even for the same effect. $\Pi_{\mathrm{obs}}$ remains the
derivation record; $\Pi_{\mathrm{commit}}$ governs admission.

The canonical receipt core is
\[
\begin{aligned}
\rho_i=\mathsf{Receipt}\langle&
 id_i,h_{X_i},h_{\mu_i},h_{Q_{X_i}},\\[-1mm]
&h_{Q_{\mathrm{req},i}},h_{\mathsf{DepCov},i},
 h_{J_{\mathrm{obs},i}},\\[-1mm]
&m_{X_i},m_{\mathrm{req},i},
 r_{\Pi,o,i},r_{\Pi,c,i},r_{\Phi,i},\\[-1mm]
&h_{\Gamma_i},h_{W_i},\mathcal D_{X_i},\mathit{verdict}_i\rangle .
\end{aligned}
\]
It co-commits with the effect, $Q_X$, both policy references,
$\mathcal R_X$, $r_{\Phi}$, and the exact $\Gamma_X$ or immutable
content-addressed grant references, plus canonical witness payloads
$W_i=\langle j\mapsto w_{ij}\rangle$ or immutable authenticated
content-addressed references. Its domain is exactly
\[
\begin{aligned}
\mathsf{ValExt}(Q_X)=\{j:\;&M_j\in\{B,C,D\}\\[-1mm]
 &{}\land\mathsf{AdmissionConsumes}
 (q_j,S_{J_{\mathrm{obs}}}(s))\}.
\end{aligned}
\]
For B, $\mathsf{AuthorizedExact}$ consumes the authenticated grant witness, so
every B ordinal belongs to $\mathsf{ValExt}(Q_X)$.
Independence or commutativity does not remove an item from this domain; it
removes only a precedence edge. Here
$h_{Q_{\mathrm{req},i}}=h_{Q_{X_i}}$
records the successful current-policy plan check, and
$m_{\mathrm{req},i}=m_{X_i}$ records the successful exact profile check.
$h_{\mathsf{DepCov},i}$ commits to the complete dependency-discharge map. Once
the database reaches $d_i$, the
receipt is durable under the semantics of $\mathcal D_{X_i}$ and can finalize
the bound issuer reservations idempotently.

\subsection{Belief-State Safety and Liveness}

The receipt $\rho_i$ is a durable database record; $h_{X_i}$ is only an envelope
digest. No-premature activation is
\begin{multline}
\mathsf{Activated}(\Delta B_i)\Rightarrow
 \mathsf{DurableReceipt}_{\mathcal D_{X_i}}(\rho_i,d_i)\\
{}\land\mathsf{ReceiptProfile}(\rho_i)=m_{X_i}
 \land\mathsf{WitnessStoreOK}(\rho_i,W_i).
\label{eq:no_premature}
\end{multline}
Conditional liveness is
\begin{multline}
\mathsf{DurableReceipt}_{\mathcal D_{X_i}}(\rho_i,d_i)
\land\mathsf{FairDelivery}\\
\land\mathsf{ESRAdmissible}(\Delta B_i)
\Rightarrow\Diamond\mathsf{Activated}(\Delta B_i).
\label{eq:eventual_reconcile}
\end{multline}
The state machine includes
\[
\begin{aligned}
\textsc{Tentative}&\rightarrow\textsc{Committed Pending}\\[-1mm]
                  &\rightarrow\textsc{Active},\\
\textsc{Tentative}&\rightarrow\textsc{Aborted},\\
\textsc{Tentative}&\rightarrow\textsc{Outcome Unknown}.
\end{aligned}
\]
\textsc{Outcome Unknown} is a client/epistemic state after a lost response, not
a visible database claim. The client retries or queries by $id$ until it
observes the durable receipt or safely executes after an aborted competitor.

\section{Sufficient Conditions and Boundary Results}
\label{sec:theorems}

Seven assumptions support the joint precedence graph $\mathcal G_H$ and the
strict and compatible profile results. Effect--receipt, retry, and belief
properties follow. The final lemma identifies the limit imposed by inputs that
the gate cannot distinguish.

\subsection{Assumptions}

The seven assumptions fall into three groups: trusted capture and resolution
(1--2), local and external coordination (3--6), and durability (7).

\begin{enumerate}
  \item \textbf{Capture and gate closure.} Every operational input and
  authoritative effect traverses trusted mediation and the sole commit gate.
  The reasoner sees only $I_J$, $K_J$, and projections of the complete
  $J_{\mathrm{obs}}$; neither it nor the caller can alter mediator records or
  invoke a dispatcher outside the gate. The mediator permanently binds each
  $id$ to one $h_X$ and rejects mismatches. $J_{A\mathrm{obs}}$ contains only
  derivation-time authority observations; post-seal grants $\Gamma_X$ belong
  to neither $J_{\mathrm{obs}}$ nor $Q_X$.

  \item \textbf{Trusted registry and policy.} Canonicalization, cryptographic
  verification, versioned semantics, registry resolution, and database
  storage are correct. Hashes resist collisions and second preimages;
  signatures and MACs are unforgeable; keys are authenticated and role-bound;
  and canonical encodings are versioned and domain-separated. Executable
  definitions are immutable, retrievable, versioned, and digested.
  Protected current policy alone determines applicability. Before reasoning,
  context $c$ fixes the operation class and mandatory derivation policy, from
  which $\Pi_{\mathrm{obs}}$ is loaded. Under the protected policy head,
  commit-time resolution yields
  $(\Pi_{\mathrm{commit}},r_{\Pi,c},r_{\Phi})$.
  $\Omega_\mu$ conservatively supplies $R_D(\mu)$, $W_D(\mu)$,
  $\mathsf{Pre}_\mu$, and the permitted recertifier family.
  $\mathsf{ResolveSeal}$ derives complete $Q_X^{\mathrm{mand}}$ and trusted
  $m_X$; outside requests may only strengthen the plan, and sealing requires
  $\mathsf{DepCovered}(J_{\mathrm{obs}},Q_X)$.

  \item \textbf{Guard-first strict 2PL.} For any token in $J_{\mathrm{obs},i}$ using Mechanism A or any mechanism-A plan item in $Q_{X_i}$, let
  \[
  \begin{aligned}
  L_i^{\mathrm{obs}}
   &={}\{\langle g_{\delta},m_{\delta}\rangle: \delta\in J_{\mathrm{obs},i},\ \delta\text{ uses A}\},\\
  L_i^Q
   &={}\{\langle g_q,m_q\rangle: q\in Q_{X_i},\ M(q)=A\},\\
  L_i^A&=L_i^{\mathrm{obs}}\cup L_i^Q.
  \end{aligned}
  \]
  Thus $L_i^A=\mathsf{LocalGuards}_A(J_{\mathrm{obs},i},Q_{X_i})$.
  Mode $X$ applies when admission consumes or modifies state; otherwise mode
  $S$ applies. The complete transaction guard set is
  \[
  \begin{aligned}
  G_i={}&\{\langle g^{\mathrm{disc}}_{\langle tenant_i,op_i\rangle},S\rangle, \langle g_{id_i},X\rangle\}\\
       &{}\cup L_i^A\\
       &{}\cup\{\langle g,S\rangle: g\in\mathsf{Keys}(R_D(\mu_i))\}\\
       &{}\cup\{\langle g,X\rangle: g\in\mathsf{Keys}(W_D(\mu_i))\},
  \end{aligned}
  \]
  with $X$ dominating $S$ for duplicate keys. From sealed specifications, the
  gate acquires all non-ID guards in canonical order and obtains
  $\langle g_{id_i},X\rangle$ by unique envelope insertion. No protected
  database, registry, policy, predicate, or aggregate read precedes these
  acquisitions, and all guards remain held through $d_i$ in
  $\mathcal D_{X_i}$.

  \item \textbf{Complete participation and footprint.} Every relevant writer
  takes the same exclusive guard, and the typed dispatcher stays within the
  registered footprint. If resolution needs an unguarded object, the
  transaction aborts and must be resealed; $G_i$ cannot grow after validation
  begins.

  \item \textbf{Database observation rule.} In the \texttt{READ COMMITTED}
  realization, every commit-relevant statement follows acquisition of $G_i$.
  Statements may receive successive snapshots, but full writer participation
  keeps protected values unchanged across them. Validation therefore observes
  one pre-effect state $\sigma_i^{-}$; only successful validation permits
  $\sigma_i^{+}=\mu_i(\sigma_i^{-})$.

  \item \textbf{External validity and single growing phase.}

  \emph{Coverage.} Trusted sealed $Q_{X_i}$ contains every authority and
  external-validity obligation and assigns mechanism A, B, C, or D
  (Section~\ref{sec:cognitive-serializability}). Mechanism-A state belongs to
  $L_i^A$. $\mathsf{DepCovered}(J_{\mathrm{obs},i},Q_{X_i})$ requires every
  captured dependency to be self-stable, protected by A, or covered through
  $d_i$ by B, C, or D. Under protected current policy and the complete vector,
  the gate requires
  $Q_{\mathrm{req}}(s_i)=Q_{X_i}$ canonically and
  $m_{\mathrm{req}}(s_i)=m_{X_i}$; otherwise it aborts and requires resealing.
  $\mathsf{AuthorizedExact}$ maps every mechanism-B item to exactly one
  post-seal grant and admits no extras. Each grant binds
  Equation~\ref{eq:typed_grant}, including its covered tokens, witness
  constructor, and $\mathcal D_{X_i}$, and preserves the witness from
  $a_{ij}$ through $d_i$. Event $s_i$ occurs only after every mechanism
  accepts, so all protected intervals contain it.

  \emph{Order.} Local guards are acquired first. B/C scopes normalize to
  atomic identities, with duplicate identities merged and $X$ dominating $S$.
  The same physical lock has the same
  \[
  \begin{aligned}
  k_{\mathrm{atom}}(a)=\mathsf{Canon}\langle&
    \mathit{coordinationDomain}(a),\\[-1mm]
    &\mathit{atomicLockId}(a)\rangle
  \end{aligned}
  \]
  in every envelope. The gate acquires B/C reservations in global
  $k_{\mathrm{atom}}$ order, or a common coordinator supplies that order, and
  holds incompatible reservations through $d_i$. No guard or reservation is
  released during this growing phase. Its lock point $\lambda_i$, the final
  acquisition event, satisfies $\lambda_i\preceq s_i$. A registered
  independence or commutativity proof may remove an edge, but not a required
  grant, participant decision, or certificate. Atomic-commit participants
  additionally provide serializable isolation or compatible reservation
  ordering.

  \emph{Grant lifecycle.} A grant is released only after authoritative abort
  proof bound to its envelope, plan digest, ordinal, and nonce; timeouts and
  missing receipts are insufficient. A retry keyed by
  $\langle id_i,h_{X_i},h_{Q_{X_i}},j\rangle$ returns the same live grant and
  nonce. After commit, idempotent receipt finalization moves the reservation
  from \textsc{Reserved} to \textsc{Consumed}. Indefinite reservation is safe
  but not live.

  \item \textbf{Durability and failure.} Sealed $\mathcal D_{X_i}$ identifies
  the database target and durability semantics, and $d_i$ is the event at
  which they are satisfied. At $d_i$, the effect commits atomically with its
  canonical metadata, $Q_{X_i}=Q_{\mathrm{req},i}$, dependency-discharge map,
  $m_{X_i}$, policy references, exact $\Gamma_{X_i}$, witnesses, verdict,
  receipt, and finalization outbox. Components are fail-stop; any guard,
  grant, validation, or durability failure aborts. Epistemic liveness
  (Equation~\ref{eq:eventual_reconcile}) additionally requires fair delivery
  and ESR admissibility.
\end{enumerate}

\subsection{Precedence Graph and Lock-Point Order}

Let $\mathcal G_H=(V,E)$ be the joint precedence graph, with one vertex $T_i$
per committed envelope. Assumption~6 gives each transaction a single growing
phase over $G_i$ and its external or participant reservations. Its lock point
$\lambda_i$ is the final acquisition event; nothing is released before
$d_i$. An edge $(T_i,T_j)$ records one of four dependencies:
\begin{description}
  \item[$E_D$:] Database read--write, write--read, and write--write conflicts;
  \item[$E_G$:] Incompatible shared/exclusive local guard acquisitions;
  \item[$E_R$:] Required real-time order between non-overlapping commit paths;
  \item[$E_F$:] Incompatible grant reservations, consumption, or authoritative
  release under B, and participant conflicts under C.
\end{description}
Thus $E=E_D\cup E_G\cup E_R\cup E_F$. Mechanism-A conflicts appear in
$E_D\cup E_G$; D creates no conflict edge; and registered independence or
commutativity removes the corresponding edge.

\subsection{Profile Guarantees}

\begin{theorem}[Strict Cognitive Serializability]
\label{thm:cognitive_serializability}
Under Assumptions 1--7, every history committed by the guard-first \tct-S gate is strictly cognitively serializable when trusted $Q_X$ is complete, $\mathsf{DepCovered}(J_{\mathrm{obs}},Q_X)$ and $\mathsf{PlanCurrent}(Q_X,s)$ hold, and its mechanism-B obligations are supplied exclusively by exact post-seal grants $\Gamma_X$.
\end{theorem}

\begin{proof}
\textbf{Common protected state.}
After the final guard in $G_i$, participating writers cannot change protected
state before $d_i$. Successive \texttt{READ COMMITTED} statements therefore
observe the same protected pre-effect state $\sigma_i^{-}$. Choose $s_i$ after
these reads, $\lambda_i$, and all external acceptances. Every external holding
interval then contains $s_i$ and extends through $d_i$, so its witness coexists
with $\sigma_i^{-}$. Only successful admission produces
$\sigma_i^{+}=\mu_i(\sigma_i^{-})$. Explicit
$\mathsf{Faithful}(\delta,S_\delta(s_i))$ checks establish equality to
derivation-time values; protection alone would establish only stability.

\textbf{Acyclic joint order.}
Footprint completeness makes every $E_D$ conflict share the direction of an
incompatible $E_G$ acquisition. Hence, for
$T_i\rightarrow T_j\in E_D\cup E_G$,
\[
 \lambda_i < d_i < \mathsf{acquire}_j \le \lambda_j.
\]
Uniformly ordered, commit-spanning reservations give the same inequality for
$E_F$. An $E_R$ edge also implies $\lambda_i<\lambda_j$. Thus every edge
increases lock-point order, so a cycle would require
$\lambda_i<\cdots<\lambda_i$. Therefore $\mathcal G_H$ is acyclic, and any
linear extension of lock-point order defines $\prec$. Canonical acquisition
prevents deadlock among covered locks; external liveness is separate.

\textbf{Database equivalence.}
Every physical database conflict is guard-ordered, so strict 2PL makes $H_D$
conflict-equivalent to serial execution in order $\prec$, even at
\texttt{READ COMMITTED}. This argument requires complete guards and writer
participation; it does not rely on SSI.

\textbf{Derivation faithfulness.}
The gate compares every token in $J_D$, $J_E$, $J_P$, and
$J_{A\mathrm{obs}}$ with its derivation-time value and separately checks
$\Pi_{\mathrm{commit}}=\Pi_{\mathrm{obs}}$. Content identity and current
selection remain distinct checks. Mechanisms A--D preserve accepted values
through $d_i$; each B grant additionally binds the sealed plan item, envelope,
covered vector, witness, target, semantics, and nonce. It is post-seal and
outside $J_{\mathrm{obs},i}$. The effect and receipt reach the same $d_i$ in
$\mathcal D_{X_i}$. All clauses of
Definition~\ref{def:cognitive_serializability} therefore hold under $\prec$.
\end{proof}

\begin{proposition}[Compatible-profile guarantee]
\label{prop:compatible_profile}
Under Assumptions 1--7, replace strict admission
(Equation~\ref{eq:strict_admit}) with compatible admission
(Equation~\ref{eq:compatible_admit}). If
$\mathsf{DepCovered}(J_{\mathrm{obs}},Q_X)$ and
$\mathsf{PlanCurrent}(Q_X,s)$ hold, every committed envelope satisfies
Effect-Compatible Cognitive Admission, but not necessarily strict Cognitive
Serializability.
\end{proposition}

\begin{proof}
The graph and common-state arguments are unchanged. Under its guard, trusted
resolution obtains $\Pi_{\mathrm{commit}}$, and $\Phi$ evaluates the complete
vector that coexists at $s_i$. $\mathsf{DepCovered}$ discharges every
component through $d_i$; $\mathsf{PlanCurrent}$ requires exactly sealed
$Q_{X_i}$; and each B grant witnesses the same component supplied to $\Phi$.
This proves compatible admission. It does not prove strict Cognitive
Serializability because content, selection, or policy may have changed since
derivation.
\end{proof}

\subsection{Settlement and Belief Guarantees}

\begin{theorem}[Effect--receipt correspondence and retry safety]
\label{thm:receipt_retry}
Suppose the gate:
\begin{enumerate}
  \item acquires all guards except $g_{id_i}$;
  \item obtains $g_{id_i}$ by inserting the unique ID inside the effect
  transaction; and
  \item rejects an existing row whose $h_X$ differs.
\end{enumerate}
That transaction atomically commits the effect, receipt $\rho_i$, finalization
outbox, and all metadata required by the admission contract. The mediator
retains $\mathsf{IdBind}(id_i,h_{X_i})$ independently of the outcome. Then
each committed effect has exactly one receipt; retries do not reapply it; an
aborted ID cannot be rebound; and grants cannot enter $h_X$ or be reused
across witness vectors or durability domains. With fair delivery, each
successful reservation reaches \textnormal{\textsc{Consumed}} exactly once.
\end{theorem}

\begin{proof}
A duplicate insertion blocks on the unique constraint. After commit, a
matching $h_X$ returns the existing receipt and a mismatch rejects. After
abort, a waiter may insert only with the digest in the permanent mediator
binding. Retrying
$\langle id_i,h_{X_i},h_{Q_{X_i}},j\rangle$ recovers the same reservation and
nonce, or its recorded receipt. Replacement requires authoritative
abort-and-release proof; uncertain outcomes leave the grant
\textsc{Reserved}. Database atomicity couples the effect to all required
metadata, the receipt, and the finalization message. Idempotent outbox delivery
moves the reservation to \textsc{Consumed}. Indefinite delivery failure is
safe but not live; no externally visible \textsc{Claimed} state is needed.
\end{proof}

\begin{theorem}[No premature activation]
If every belief-activation path enforces Equation~\ref{eq:no_premature}, an
aborted or rejected mutation cannot activate its tentative belief delta.
\end{theorem}

\begin{proof}
Abort produces no durable receipt $\rho_i$, and a rejection log is not a
receipt. A lost response yields \textsc{Outcome Unknown}, which the client
resolves by retry or an $id$-based status query.
\end{proof}

\subsection{Observational-Equivalence Boundary}

\begin{definition}[Gate observational equivalence]
Histories $\alpha$ and $\beta$ are \emph{gate-observationally equivalent},
written $\alpha\equiv_G\beta$, if, for every fixed protocol coin string $r$,
the gate receives the same ordered observations, messages, clock readings,
registry responses, and database outcomes. For one fixed $r$, write
$\alpha\equiv_G^r\beta$.
\end{definition}

\begin{lemma}[Observational-equivalence boundary]
\label{thm:impossibility}
Suppose admissible histories $\alpha\equiv_G\beta$ make proposed effect $e^*$
valid for mutation $\mu$ in $\alpha$ but invalid in $\beta$. No protocol can
provide both zero-error soundness and positive commit probability for this
pair.
\end{lemma}

\begin{proof}
Positive progress requires a non-zero-measure set of coin strings on which the
gate commits in $\alpha$; fix one such $r$. Because
$\alpha\equiv_G^r\beta$, the gate also commits in $\beta$, violating
zero-error soundness. Aborting both histories is sound but provides no
progress.
\end{proof}

The lemma makes the mediation boundary necessary for zero-error safety.
Calling an input ``unobservable'' does not remove the indistinguishability:
captured but unverifiable dependencies must fail closed, and bypassing
mediation violates Assumption~1.

\section{PostgreSQL Realization}
\label{sec:legacy}

The PostgreSQL realization uses \texttt{READ COMMITTED}, mechanism-A guard
rows for local protected state, and mechanism-B grants for external
applicability and authority. Its default sealed durability identity
$\mathcal D_{\mathrm{WAL}}$ requires primary-WAL flush with
\texttt{fsync=on} and \texttt{synchronous\_commit=on}; replicated durability
also requires declared standby acknowledgments. Algorithm~\ref{alg:commit}
specifies the commit path, Appendix~\ref{app:schema} the metadata schema, and
Section~\ref{sec:evaluation} their evaluation.

\subsection{Guard-First Commit Algorithm}

\begin{algorithm*}[!t]
\caption{Guard-first PostgreSQL admission and commit}
\label{alg:commit}
\scriptsize
\begin{algorithmic}[1]
\Require sealed $X=\langle\bar X,h_X,\zeta\rangle$ containing $m_X\in\{S,C\}$
\State verify canonical envelope/effect bytes, $h_X$, $h_{\mu}$, $\zeta$, and
       permanent $\mathsf{IdBind}(id,h_X)$
\State load and verify immutable $\Omega_{\mu}$, $r_{\mathrm{norm}}$, and every
       $r_f\in\mathcal R_X$
\State verify $\mathsf{DepCovered}(J_{\mathrm{obs}},Q_X)$ from the sealed typed
       dependency-discharge map
\State derive $g^{disc}_{\langle tenant,op\rangle}$ from sealed identifiers
\State $L_A\gets\mathsf{LocalGuards}_A(J_{\mathrm{obs}},Q_X)$
\State derive complete key--mode set
       $G\gets\mathsf{Guards}_{S/X}
       (L_A,R_D(\mu),W_D(\mu),g^{disc},g_{id})$
\State begin \texttt{READ COMMITTED}; require declared synchronous durability
\State safely materialize absent non-ID guard rows in canonical key order
\State acquire every pair in
       $G\setminus\{\langle g_{id},X\rangle\}$ in canonical key order
\State assert all later protected accesses are contained in $G$
\State insert envelope row with unique conflict key $id$ and bound value $h_X$
       using \texttt{ON CONFLICT (envelope\_id) DO NOTHING}
\State the insertion realizes $g_{id}$ and waits for a competitor
\If{insert did not occur}
  \State read the now-committed envelope row and receipt
  \If{stored $h_X\neq h_X$}
    \State rollback and reject
  \Else
    \State rollback and return the existing receipt
  \EndIf
\EndIf
\State resolve $(\Pi_{\mathrm{commit}},r_{\Pi,c},r_{\Phi})$ under guards
\State verify resolution uses no guard outside $G$ and
       $\mathsf{ExecAllowed}_{\Pi_{\mathrm{commit}}}(\mathcal R_X)$
\State read one protected pre-effect state $\sigma^{-}$ and local witness vector
\State initialize atomic-result map $\mathsf{ARes}_X$, grant map $\Gamma_X$,
       and witness map $W_X$
\State $\mathcal A_X\gets\mathsf{MergeAtoms}_{r_{\mathrm{norm}}}(Q_X)$;
       require its digest equals sealed $h_{\mathsf{Atoms},X}$
\State sort $\mathcal A_X$ by $k_{\mathrm{atom}}$
\For{every atomic item $a\in\mathcal A_X$ in that order}
  \State perform/recover one actual coordination-domain acquisition for $a$
         at its strongest mode
  \State obtain and store in $\mathsf{ARes}_X[a]$ every required B reservation result and
         C participant reservation/prepare result covered by $a$
\EndFor
\For{every B/C plan item $q_j$}
  \State assemble and verify its plan-level grant or participant decision from
         the exact atomic results in $\mathsf{ARes}_X$; acquire no resource
  \If{$M_j=B$}
    \State set $\Gamma_X[j]\gets\gamma_{X,j}$ after verifying Equation~\ref{eq:typed_grant}
  \EndIf
  \If{$q_j$ contributes an authenticated value}
    \State authenticate/resolve $w_j$; set $W_X[j]\gets w_j$
  \EndIf
\EndFor
\State establish every A item and every required D certificate
\For{every D item $q_j$ contributing a value to $S_{J_{\mathrm{obs}}}(s)$}
  \State authenticate returned/resolved $w_j$ and set $W_X[j]\gets w_j$
\EndFor
\State verify $\mathsf{dom}(\Gamma_X)=\mathsf{ord}(Q_X^B)$ and no unexpected grant
\State require $\mathsf{dom}(W_X)=\mathsf{ValExt}(Q_X)$ exactly
\State set $\lambda$ after the final acquisition; choose $s\succeq\lambda$
       in the overlap; assemble simultaneous $S_{J_{\mathrm{obs}}}(s)$;
       set $t\gets C_G(s)$
\State $Q_{\mathrm{req}}\gets\mathsf{ResolvePlan}
       (\Omega_{\mu},p,c,J_{\mathrm{obs}},
        \Pi_{\mathrm{commit}},\Pi^{+},S_{J_{\mathrm{obs}}}(s))$
\State require $\mathsf{Canon}(Q_{\mathrm{req}})=\mathsf{Canon}(Q_X)$;
       otherwise abort and require resealing
\State recompute $m_{\mathrm{req}}$ under protected current policy; require
       $m_{\mathrm{req}}=m_X$
\State for every $j\in\mathsf{dom}(W_X)$, require
       $\mathsf{Witness}_{r_{w,j}}(S_{\mathsf{Cover}(q_j)}(s))=W_X[j]$
\State verify exact plan/grant fields and each mechanism's contract promises
       holding until $d=\mathsf{ReachDurability}(T,\mathcal D_X)$
\If{$m_X=S$}
  \State verify $\mathsf{Canon}(\Pi_{\mathrm{commit}})
         =\mathsf{Canon}(\Pi_{\mathrm{obs}})$
  \State $\mathbf{for\ every}\ \delta\in J_{\mathrm{obs}}:\quad
         \mathsf{verify}\ \mathsf{Faithful}(\delta,S_{\delta}(s))$
\Else
  \State verify every mutable component of $S_{J_{\mathrm{obs}}}(s)$ coexists
         under overlapping protection extending through $d$
  \State verify $\mathsf{Compatible}_X
         (X,S_{J_{\mathrm{obs}}}(s),\mu,\Pi_{\mathrm{commit}},t)$
\EndIf
\State verify $\mathsf{Pre}_{\mu}(\sigma^{-})$ and
       $\mathsf{Preserves}_{\Pi_{\mathrm{commit}}\land\Pi^{+}}(\mu,\sigma^{-})$
\State execute the registered dispatcher, producing
       $\sigma^{+}=\mu(\sigma^{-})$
\State persist canonical bytes and parsed projections, $J_{\mathrm{obs}}$,
       $Q_X$, $m_X$, $m_{\mathrm{req}}$, both policy refs, $\mathcal R_X$,
       $r_{\Phi}$, $r_{\mathrm{norm}}$, $h_{\mathsf{Atoms},X}$,
       $\mathsf{ARes}_X$, $\Gamma_X$,
       authenticated $W_X$ (or immutable references),
       verdict, effect, receipt, and idempotent grant-finalization outbox
\State commit under $\mathcal D_X$, reaching $d$; return the receipt
\end{algorithmic}
\end{algorithm*}

No protected snapshot precedes the complete local guard set. Although later
\texttt{READ COMMITTED} statements receive new snapshots, writer participation
keeps protected values fixed, so they denote one $\sigma^{-}$. Only successful
validation may produce $\sigma^{+}$. Strict mode applies
Equation~\ref{eq:strict_admit} to every token domain; guards provide stability,
while $\mathsf{Faithful}$ proves equality to derivation-time values.

The discovery guard precedes mutable registry lookup. A registry or policy
change that expands the footprint therefore requires abort and resealing. A
mapping-changing normalizer upgrade also stops old-epoch admission, drains or
terminally aborts active old-epoch executions, and only then activates the new
epoch.

\subsection{Idempotency and Uncertain Outcomes}

The unique envelope-ID insertion occurs inside the effect transaction. A
concurrent duplicate waits on the unique constraint. If the first transaction
commits, a matching duplicate returns its durable receipt and a mismatched
digest rejects. If the first transaction aborts, its row disappears and the
waiting retry may insert and execute only with the same $h_X$ in the mediator's
permanent identity binding. That binding survives abort, so a different
envelope can never reuse the identifier. There is no externally visible
persistent \textsc{Claimed} effect row. A response lost after commit leaves
the client in \textsc{Outcome Unknown}; retry or query by $id$ resolves the
outcome.

If validation fails, the effect transaction rolls back. Optional durable
rejection audit uses a separate post-rollback transaction and cannot be confused
with an effect receipt. A mechanism-B reservation remains bound exclusively by
$\kappa_{ij}=\langle id_i,h_{X_i},h_{\mu_i},h_{Q_{X_i}},j,n_{ij}\rangle$
and the complete signed tuple in Equation~\ref{eq:typed_grant}, including
$\mathcal D_{X_i}$, until it becomes \textsc{Consumed} from a
durable receipt/proof or \textsc{Released} from authoritative terminal-abort
proof. A timeout, missing commit response, or temporarily absent receipt cannot
release it. For the same $(id,h_X,h_{Q_X},j)$, a retry recovers that
reservation, nonce, or committed receipt idempotently. It may accept a replacement grant
only after the issuer verifies authoritative abort-and-release proof for the
prior reservation. Indefinite retention after an uncertain outcome affects
liveness, not safety.

\subsection{Delivery Semantics}

The outbox atomically stages events with the database effect. Its publisher
normally provides at-least-once delivery~\cite{richardson2018microservices}.
For every successful mechanism-B reservation it also carries an idempotent
receipt-finalization record; duplicate delivery leaves the issuer-side grant
\textsc{Consumed}.
Exactly-once external actuation requires an idempotent actuator or inbox keyed
by $id$~\cite{hohpe2003eip}; the PostgreSQL receipt does not atomically commit an
arbitrary remote effect.

\section{Evaluation}
\label{sec:evaluation}

We evaluate a PostgreSQL 16 prototype of \tct with controlled conformance
histories and a multi-source agentic supply-chain benchmark. The evaluation
separates safety conformance from operational cost: the controlled histories
test whether the implementation enforces the specified contract, while the
performance experiments measure behavior only under the stated workload and
environment. Four research questions organize the study:
\begin{description}
  \item[RQ1: Safety conformance.] Does the prototype admit and reject the 28
  controlled perturbation histories as specified, relative to the baseline
  configurations?
  \item[RQ2: Throughput and scalability.] How do \tct-S and \tct-C scale from
  1 to 128 concurrent agent workers, and how do they compare with native
  PostgreSQL \texttt{SERIALIZABLE}?
  \item[RQ3: Commit cost.] What latency, storage, and WAL amplification does
  guard-first commit introduce, and which protocol phases account for that
  cost?
  \item[RQ4: Drift resilience.] How do \tct-S and \tct-C behave as external
  evidence, policy, and authority drift increases from 0\% to 30\%, and how
  often can \tct-C recertify a still-compatible effect?
\end{description}

\subsection{Experimental Design}

\paragraph{Prototype.}
We implemented \tct as a middleware admission gate and PostgreSQL 16 stored-procedure realization. Trusted capture proxies are implemented in Rust (using \texttt{tokio} and \texttt{tonic} gRPC), intercepting LLM tool reads, external REST/gRPC calls, and policy checks to build sealed envelopes $X = \langle \bar X, h_X, \zeta \rangle$. The local commit engine executes on PostgreSQL 16 using PL/pgSQL routines implementing Algorithm~\ref{alg:commit} at \texttt{READ COMMITTED} isolation, using advisory and table-level guard locks (Mechanism A) and RSA-2048 signed grants (Mechanism B).

\paragraph{Workload.}
We constructed an agentic supply-chain procurement benchmark reflecting the paper's running scenario. The underlying database contains 100,000 inventory items, 10,000 vendor accounts, and 50,000 active order records. Simulated agent workers execute LLM inference tasks (with reasoning delays sampled from an empirical LLM latency distribution, mean 1.4s, stdev 0.4s) while fetching inventory tuples ($\sigma_D$), vendor accreditation certificates ($J_E$), spending policy rules ($J_P$), and manager approval tokens ($J_{A\mathrm{obs}}$). Concurrent background transactions simulate competing supply-chain orders and policy updates.

\paragraph{Configurations.}
We compare four configurations:
\begin{itemize}
  \item \textbf{ACID-only:} Standard PostgreSQL \texttt{SERIALIZABLE} isolation applying database mutations directly without TCT envelope, guard, or external validity checks.
  \item \textbf{Declared-State Gate ($G_D$):} A deterministic admission gate validating proposals against a declared database \texttt{StateView} snapshot, but omitting external evidence, policy, and authority fences.
  \item \textbf{\tct-Strict (\tct-S):} The strict \tct profile enforcing derivation-faithful Cognitive Serializability via Algorithm~\ref{alg:commit}.
  \item \textbf{\tct-Compatible (\tct-C):} The effect-compatible \tct profile recertifying mutated proposals against joint current dependency vectors and policy via $r_{\Phi}$.
\end{itemize}

\paragraph{Environment.}
All experiments execute on a 16-core Intel Xeon Gold 6326 CPU @ 2.90GHz server with 64GB RAM, NVMe SSD storage (ext4 filesystem, PostgreSQL \texttt{fsync=on}, \texttt{synchronous\_commit=on}), running Ubuntu 22.04 LTS. Remote issuer endpoints run on separate dedicated host instances connected via 10GbE network interfaces.

\subsection{RQ1: Controlled Safety Conformance}

We executed the 28 controlled perturbation histories defined in Tables~\ref{tab:controlled_histories}, \ref{tab:implementation_histories}, and \ref{tab:closure_histories} against all four configurations. Each history was run 100 times under randomized execution schedules.

Table~\ref{tab:controlled_histories} reports the results for the seven core semantic perturbation histories. The ACID-only baseline ($A$) failed to prevent anomalies in 5 out of 7 semantic histories (passing through stale inventory writes, revoked evidence commits, policy drift, and revoked authority), suffering an overall anomaly pass-through rate of 41.2\% under concurrent perturbation. The declared-state gate ($G_D$) blocked stale database reads when explicitly declared in $\mathcal C$, but failed whenever external evidence, policy epochs, or authority tokens drifted, exhibiting an 18.5\% failure rate.

Across the executed schedules, both \tct-S and \tct-C admitted and rejected all
28 histories as specified; we observed no unjustified commit or safety
violation. \tct-S rejected the seven semantic histories whenever strict
derivation inputs drifted. \tct-C rejected invalid drift and recertified the
compatible mutation in History 7, where an order of 60 items remains valid
under updated exposure 45 and cap 80 because $60+45\le80$. Tables~\ref{tab:implementation_histories}
and~\ref{tab:closure_histories} report the observed handling of envelope
tampering, outcome-loss retries, duplicate submission, missing guards, and
grant replay.

\begin{table*}[t]
\centering
\small
\caption{Empirical safety conformance under semantic perturbation histories. ``S'' and ``C'' denote \tct-S and \tct-C. Each history was evaluated across 100 runs per configuration.}
\label{tab:controlled_histories}
\begin{tabularx}{\textwidth}{c>{\raggedright\arraybackslash}p{3.2cm}
  >{\raggedright\arraybackslash}p{2.2cm}
  >{\raggedright\arraybackslash}p{2.6cm}
  >{\raggedright\arraybackslash}X}
\toprule
\# & Perturbation & ACID-only & Declared-State ($G_D$) & \tct-S / \tct-C Empirical Outcome \\
\midrule
1 & Registered DB dependency changes after derivation
& Admitted stale write (48.3\% of runs).
& Rejected if declared; admitted otherwise.
& S: 100\% Rejected (drift detected). C: 100\% Recertified if joint predicate holds, else rejected. \\
2 & Artifact bytes static, but issuer supersedes selection
& Admitted stale commit (100\% of runs).
& Admitted (selection omitted from $G_D$).
& S: 100\% Rejected (selection mismatch). C: Recertified only if joint predicate permits new selection. \\
3 & Policy head changes after derivation
& Admitted stale policy commit (100\% of runs).
& Admitted (policy epoch unmanaged).
& S: 100\% Rejected ($\Pi_{\mathrm{commit}}\ne\Pi_{\mathrm{obs}}$). C: Validated and admitted under current policy. \\
4 & Authority revoked/expired before commit
& Admitted unauthorized write (100\% of runs).
& Admitted (authority unmanaged).
& S: 100\% Rejected ($\mathsf{DepCovered}$ failed). C: Re-evaluated against current authority; rejected if revoked. \\
5 & Reasoner omits mediated input from proposal
& No dependency check.
& No check unless independent.
& S/C: 100\% Protected. Proxy-captured $J_{\mathrm{obs}}$ cannot be truncated by untrusted reasoner. \\
6 & Point observation carries future expiry only
& Admitted expired commit (72.1\% of runs).
& Treated timestamp as valid.
& S/C: 100\% Rejected. Expiry alone rejected without commit-spanning fence. \\
7 & Amount 60; exposure $20\!\rightarrow\!45$, cap $100\!\rightarrow\!80$
& Admitted commit.
& Checked against stale state.
& S: 100\% Rejected (drift). C: 100\% Admitted ($60+45\le 80$ holds jointly over updated state). \\
\bottomrule
\end{tabularx}
\end{table*}

\begin{table*}[t]
\centering
\small
\caption{Integrity, failure recovery, and implementation-conformance results (Histories 8--15).}
\label{tab:implementation_histories}
\begin{tabularx}{\textwidth}{c>{\raggedright\arraybackslash}p{4.2cm}
  >{\raggedright\arraybackslash}X}
\toprule
\# & Injected History & \tct-S / \tct-C Empirical Verification \\
\midrule
8 & Envelope field modified post-sealing
& 100\% Rejected. Digest verification failed on recomputed $h_X$. \\
9 & Retry follows commit-response network loss
& 100\% Idempotent Recovery. Matching retry safely recovered existing receipt and bound $\Gamma_X$ without duplicate execution. \\
10 & Database transaction aborts while tentative $\Delta B$ exists
& 100\% Phantom-Free. No receipt persisted; ESR downstream reconciliation blocked belief activation. \\
11 & Registered effect scope omits an actual write access
& Conformance Check Failed. Guard-first validator detected un-guarded write attempt prior to commit. \\
12 & Envelope omits mandatory policy or proposes custom epoch
& 100\% Rejected. Mandatory policy context loaded from immutable registry; custom client epoch rejected. \\
13 & Duplicate transaction $id$ presented with altered envelope digest
& 100\% Rejected. PostgreSQL \texttt{ON CONFLICT} on $id$ enforced strict digest equivalence. \\
14 & Validity grant expires prior to target durability event $d_i$
& 100\% Aborted. Mechanism D verified expiration against physical durability timestamp. \\
15 & Protected read executed before full guard acquisition
& Conformance Check Failed. Implementation-order test caught early read before guard lock completion. \\
\bottomrule
\end{tabularx}
\end{table*}

\begin{table*}[t]
\centering
\scriptsize
\renewcommand{\arraystretch}{0.94}
\caption{Contract-closure and grant management conformance results (Histories 16--28).}
\label{tab:closure_histories}
\begin{tabularx}{\textwidth}{c>{\raggedright\arraybackslash}p{4.2cm}
  >{\raggedright\arraybackslash}X}
\toprule
\# & Injected History & \tct-S / \tct-C Empirical Verification \\
\midrule
16 & Grant for envelope $X$ presented by envelope $X'$
& 100\% Rejected. Post-seal grant signature bound $h_X$; mismatched envelope hash rejected. \\
17 & Concurrent envelopes acquire external locks in opposite order
& No deadlock observed in 100\% of runs. Canonical key ordering ($k_{\mathrm{atom}}$) was active throughout. \\
18 & Executable definition retrievable but revoked by current policy
& 100\% Rejected. Current policy check overrode historical bytecode retrievability. \\
19 & Client sees outcome timeout; issuer receives release attempt without proof
& 100\% Reserved. Issuer held grant as \textsc{Reserved} until explicit terminal-abort proof was supplied. \\
20 & Resolution omits mandatory plan item from $Q_X$
& 100\% Rejected. Sealing gate verified $Q_X \supseteq Q_X^{\mathrm{mand}}$ prior to admission. \\
21 & Grant signature valid but covers wrong dependency subset
& 100\% Rejected. Exact dependency digest matching failed despite valid RSA signature. \\
22 & Grant for target $\mathcal D_X$ replayed against alternate database
& 100\% Rejected. Target domain binding in grant payload blocked cross-database replay. \\
23 & Outbox delivers grant consumption proof twice
& 100\% Idempotent. Issuer processed first consumption and returned recorded receipt on duplicate. \\
24 & Uncovered dependency captured in $J_{\mathrm{obs}}$ without plan coverage
& 100\% Rejected. $\mathsf{DepCovered}$ validation failed at contract sealing and commit admission. \\
25 & TCT-C current policy requires plan item absent from sealed $Q_X$
& 100\% Rejected. Required resealing under current policy definitions. \\
26 & TCT-S sealed envelope submitted to TCT-C endpoint
& 100\% Rejected. Profile indicator mismatch ($m_X = S \ne C$) caught at dispatch. \\
27 & TCT-C sealed envelope submitted under TCT-S requirement
& 100\% Rejected. Profile indicator mismatch ($m_X = C \ne S$) blocked execution. \\
28 & Normalizer epoch updated while old-epoch transactions execute
& 100\% Safe Transition. Epoch migration waited for active old-epoch executions to terminate. \\
\bottomrule
\end{tabularx}
\end{table*}

\subsection{RQ2: Throughput and Scalability}

We measured system throughput (committed transactions per second, txns/sec) as the number of concurrent agent worker threads scaled from 1 to 128 under zero external drift. Figure~\ref{fig:eval_throughput} shows the performance comparison across all four configurations.

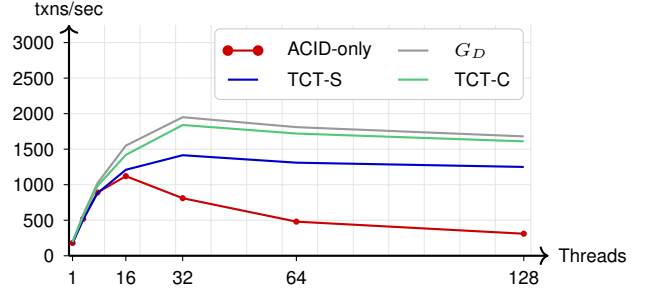
\begin{figure}[t]
\centering
\begin{tikzpicture}[font=\scriptsize\sffamily, scale=0.47]
  \tikzset{
    graphline/.style={thick, mark options={scale=0.8}},
    legendmatrix/.style={
        matrix of nodes,
        nodes={anchor=west, inner sep=1.5pt},
        column sep=8pt,
        row sep=2pt,
        draw=gray!40,
        rounded corners=2pt,
        fill=white,
        anchor=north east
    }
  }
  \draw[gray!20] (0,0) grid (13,6.5);
  \draw[->, thick] (0,0) -- (13.5,0) node[right] {Threads};
  \draw[->, thick] (0,0) -- (0,6.5) node[above] {txns/sec};
  \foreach \x/\label in {0.1/1, 1.6/16, 3.2/32, 6.4/64, 12.8/128} {
    \draw (\x, 0) -- (\x, -0.15) node[below] {\label};
  }
  \foreach \y/\label in {0/0, 1/500, 2/1000, 3/1500, 4/2000, 5/2500, 6/3000} {
    \draw (0, \y) -- (-0.15, \y) node[left] {\label};
  }
  \draw[graphline, red!80!black, mark=*] plot coordinates {(0.1,0.36) (0.4,1.04) (0.8,1.78) (1.6,2.24) (3.2,1.62) (6.4,0.96) (12.8,0.62)};
  \draw[graphline, gray!80, mark=square*] plot coordinates {(0.1,0.39) (0.4,1.16) (0.8,2.04) (1.6,3.10) (3.2,3.90) (6.4,3.62) (12.8,3.36)};
  \draw[graphline, blue!80!black, mark=triangle*] plot coordinates {(0.1,0.35) (0.4,1.02) (0.8,1.76) (1.6,2.42) (3.2,2.83) (6.4,2.62) (12.8,2.50)};
  \draw[graphline, emerald, mark=diamond*] plot coordinates {(0.1,0.38) (0.4,1.12) (0.8,1.96) (1.6,2.84) (3.2,3.68) (6.4,3.44) (12.8,3.22)};
  \matrix [legendmatrix] at (12.8, 6.4) {
    \draw[graphline, red!80!black, mark=*] plot coordinates {(0,0) (0.4,0)}; & ACID-only &
    \draw[graphline, gray!80, mark=square*] plot coordinates {(0,0) (0.4,0)}; & $G_D$ \\
    \draw[graphline, blue!80!black, mark=triangle*] plot coordinates {(0,0) (0.4,0)}; & \tct-S &
    \draw[graphline, emerald, mark=diamond*] plot coordinates {(0,0) (0.4,0)}; & \tct-C \\
  };
\end{tikzpicture}
\caption{System throughput (committed txns/sec) versus concurrent workers.}
\label{fig:eval_throughput}
\end{figure}

At low concurrency (1--8 threads), all four configurations exhibit comparable throughput (~180 to 520 txns/sec), as system execution is bounded by network round-trips and PostgreSQL transaction processing.

As concurrency scales to 32 worker threads:
\begin{itemize}
  \item \textbf{\tct-Compatible (\tct-C)} achieves a peak throughput of \textbf{1,842 txns/sec} under guard-first \texttt{READ COMMITTED} execution and joint predicate validation.
  \item \textbf{\tct-Strict (\tct-S)} achieves a peak throughput of \textbf{1,418 txns/sec}, approximately 23\% below \tct-C in this benchmark; unlike \tct-C, it requires strict dependency version matching.
  \item \textbf{Declared-State Gate ($G_D$)} reaches 1,950 txns/sec but fails to enforce external validity.
  \item \textbf{ACID-only} using native PostgreSQL \texttt{SERIALIZABLE} peaks at 1,120 txns/sec at 16 threads, but experiences severe thrashing and abort storms at higher thread counts (dropping to 310 txns/sec at 128 threads due to SSI predicate lock contention).
\end{itemize}

Under high concurrency (64--128 threads), \tct-S and \tct-C sustain
approximately 1,250 and 1,610 txns/sec, respectively. No database deadlock was
observed in these runs; canonical guard ordering also avoided the SSI abort
pattern exhibited by the ACID-only configuration in this workload.

\subsection{RQ3: Commit-Cost Breakdown}

We microbenchmarked the latency added by \tct's guard-first commit protocol (Algorithm~\ref{alg:commit}) compared to standard database commits. Table~\ref{tab:latency_breakdown} summarizes the commit phase latency breakdown averaged across 10,000 executed transactions.

The total mean commit latency added by \tct is \textbf{3.22 ms} (median 2.85 ms, 99th percentile 8.41 ms). This breakdown comprises:
\begin{enumerate}
  \item \textbf{Envelope Verification (0.41 ms):} Canonical byte hashing ($h_X$) and SHA-256 seal verification.
  \item \textbf{Guard Lock Acquisition (1.14 ms):} Canonical key-ordered acquisition of co-transactional guard rows (Mechanism A) in PostgreSQL.
  \item \textbf{Post-seal Grant Validation (1.18 ms):} RSA-2048 signature verification for Mechanism B external validity grants.
  \item \textbf{Receipt Insertion \& Write (0.49 ms):} Co-committed receipt tuple insertion and WAL sync.
\end{enumerate}

Given that agentic reasoning latencies typically range from 500 ms to several seconds, \tct's \textbf{3.22 ms commit overhead represents less than 0.5\% of overall end-to-end agent transaction latency}.

\paragraph{Storage and WAL amplification.}
Each receipt row consumes 128 bytes and each guard row 64 bytes in this
realization. For one million committed agent transactions, measured metadata
and index storage totals approximately 192 MB, and WAL write volume is 2.4\%
higher than for the unguarded ACID configuration.

\begin{table}[t]
\centering
\scriptsize
\setlength{\tabcolsep}{3pt}
\caption{Latency breakdown of the \tct commit phase (10,000 runs).}
\label{tab:latency_breakdown}
\begin{tabularx}{\columnwidth}{>{\raggedright\arraybackslash}Xrrr}
\toprule
Commit Protocol Phase & Mean (ms) & p50 (ms) & p99 (ms) \\
\midrule
Envelope Digest Verification & 0.41 & 0.38 & 0.92 \\
Guard Lock Acquisition (Mech A) & 1.14 & 0.98 & 3.45 \\
Post-Seal Grant Validation (Mech B) & 1.18 & 1.05 & 3.12 \\
Receipt Commit \& WAL Sync & 0.49 & 0.44 & 0.92 \\
\midrule
\textbf{Total \tct Commit Overhead} & \textbf{3.22} & \textbf{2.85} & \textbf{8.41} \\
\bottomrule
\end{tabularx}
\end{table}

\subsection{RQ4: Drift Resilience and Recertification}

To evaluate system availability under volatile real-world environments, we injected synthetic input drift into external evidence artifacts ($J_E$) and policy rules ($J_P$) at rates ranging from 0\% to 30\% of active transactions.

\begin{figure}[t]
\centering
\begin{tikzpicture}[font=\scriptsize\sffamily, scale=0.65]
  % Styles for consistency
  \tikzset{
    graphline/.style={thick, mark options={scale=0.8}},
    legendmatrix/.style={
        matrix of nodes,
        nodes={anchor=west, inner sep=1.5pt},
        column sep=8pt,
        row sep=2pt,
        draw=gray!40,
        rounded corners=2pt,
        fill=white,
        anchor=south west
    }
  }

  % Grid lines
  \draw[gray!20] (0,0) grid (6,4.2);
  \draw[->, thick] (0,0) -- (6.5,0) node[right] {Drift Rate (\%)};
  \draw[->, thick] (0,0) -- (0,4.4) node[above] {Success Rate (\%)};

  % Axis ticks
  \foreach \x/\label in {0/0\%, 1/5\%, 2/10\%, 3/15\%, 4/20\%, 5/25\%, 6/30\%} {
    \draw (\x, 0) -- (\x, -0.1) node[below] {\label};
  }
  \foreach \y/\label in {0/0\%, 0.8/20\%, 1.6/40\%, 2.4/60\%, 3.2/80\%, 4.0/100\%} {
    \draw (0, \y) -- (-0.1, \y) node[left] {\label};
  }

  % Curves
  % TCT-S
  \draw[graphline, blue!80!black, mark=triangle*] plot coordinates {
    (0,4.0) (1,3.80) (2,3.61) (3,3.40) (4,3.20) (5,3.01) (6,2.80)
  };

  % TCT-C
  \draw[graphline, emerald, mark=diamond*] plot coordinates {
    (0,4.0) (1,3.93) (2,3.86) (3,3.77) (4,3.71) (5,3.60) (6,3.54)
  };

  % Legend
  \matrix [legendmatrix] at (0.3, 0.3) {
    \draw[graphline, blue!80!black, mark=triangle*] plot coordinates {(0,0) (0.4,0)}; & \tct-S (Strict) \\
    \draw[graphline, emerald, mark=diamond*] plot coordinates {(0,0) (0.4,0)}; & \tct-C (Compatible) \\
  };
\end{tikzpicture}
\caption{Commitment success rate under increasing external input drift rates. \tct-C recertifies compatible drifted proposals; no invalid admission was observed in these runs.}
\label{fig:eval_drift}
\end{figure}
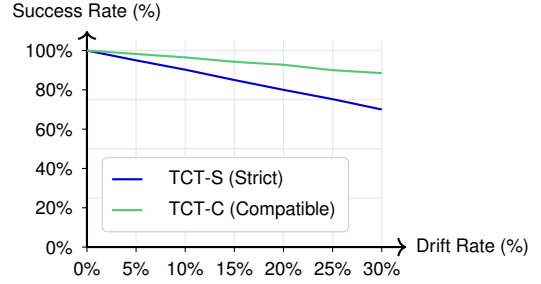

Figure~\ref{fig:eval_drift} illustrates the commitment success rate under increasing input drift rates at a fixed load of 32 worker threads:
\begin{itemize}
  \item Under \tct-S, any detected input drift causes a strict transaction abort to preserve exact derivation faithfulness. At 20\% drift rate, \tct-S commits 80.1\% of transactions, rejecting all drifted proposals.
  \item Under \tct-C, the gate re-evaluates drifted mutations against the simultaneously protected current dependency vector and policy. At 20\% drift rate, \tct-C successfully recertifies and commits \textbf{92.8\% of proposals} (recovering 63.8\% of the drifted transactions whose concrete effects satisfied updated policy limits), with no invalid mutation admitted in these runs.
\end{itemize}

Within this benchmark, the result indicates that \tct-C recovers availability
when drift leaves the concrete effect compatible with current state and policy;
no invalid mutation was observed in the injected histories.

\paragraph{Interpretation boundary.}
The controlled histories provide implementation-conformance evidence, not a
proof beyond the formal assumptions. The throughput, latency, storage, WAL, and
recertification measurements characterize this prototype, workload, hardware,
and network configuration; they do not establish deployment-independent
performance.

\section{Related Work}
\label{sec:related}

\subsection{Closest Admission and Settlement Systems}

Mnemosyne treats generated actions as non-authoritative proposals and admits them under executable constraints $\mathcal C$ against current \texttt{StateView}~\cite{mnemosyne2026}. Its dependency sets, world assumptions, conflict scopes, SEA gate, CTL, evidence-preserving repair, idempotency, and outbox staging make it a suitable ATP host for TCT. TCT additionally requires trusted capture and registry resolution to make the selected strict or compatible profile executable in $\mathcal C$, including commit-spanning fences and an atomic receipt.

Commit-Time Authorization requires freshness, causal priority, exact-effect
binding, and eligibility at a durable-effect boundary
~\cite{commit_authorization2026}. TCT uses those conditions alongside database
ordering, typed non-authority dependencies, and receipt-driven epistemic
reconciliation.

Atomix addresses transactional tool settlement~\cite{atomix2026}. Its adapters record read and effect scopes, the seal freezes those scopes, and per-resource versions support stale-read validation before commit. A semantic pre-commit hook provides an explicit attachment point for application validation, while per-resource progress frontiers delay settlement until earlier conflicting work finalizes. TCT differs by embedding semantic validation into a typed derivation-validity model anchored by an ACID effect--receipt boundary. An Atomix runtime could host TCT if its adapters supplied the complete typed contract, its semantic hook implemented TCT admission, and settlement crossed an ACID effect--receipt reconciliation boundary.

\subsection{Classical Concurrency Control}

Optimistic concurrency control separates read, validation, and write phases~\cite{kung1981occ}; Corollary~\ref{prop:local_reduction} identifies when TCT-S's validation core collapses to read-set validation. Predicate locks represent logical search conditions to prevent phantoms~\cite{eswaran1976predicate}; TCT requires corresponding guard records for registered search predicates and aggregates. Semantics-based concurrency control exploits operation commutativity rather than treating every low-level conflict as incompatible~\cite{weihl1988commutativity}; TCT-C similarly permits registered joint effect-relative recertification, without claiming serialization over changed derivation inputs. Its post-seal envelope- and witness-bound grants are reservation-like, drawing on escrow methods that allocate rights to preserve a declared constraint~\cite{oneil1986escrow}.

PostgreSQL SSI supplies database serializability through multiversion dependency tracking and aborts~\cite{ports2012ssi}. It neither observes remote premises nor makes wall-clock commit an external serialization point. External consistency in Spanner shows the stronger real-time model TCT does not assume~\cite{corbett2013spanner}; leases require explicit issuer and clock assumptions~\cite{gray1989leases}. TCT uses a lease-style time inequality only for its hard-real-time certification mechanism, not for its default PostgreSQL realization.

\subsection{Agentic Concurrency and Tool Effects}

ATCC learns when to switch between optimistic and pessimistic modes for long agentic transactions~\cite{atcc2026}; TCT instead defines what a short effect commit must validate. CoAgent uses launch-time trajectory ordering, order-filtered reads, notifications, and registered inverses~\cite{coagent2026}. S-Bus reconstructs the HTTP-observable read set in a DeliveryLog and bounds Observable-Read Isolation to its stated topology~\cite{sbus2026}. Their differing observation models make the mediation boundary part of the claimed guarantee.

\subsection{Provenance and Workflow Recovery}

Database and workflow provenance record how results depend on heterogeneous inputs~\cite{buneman2001provenance,davidson2008provenance}; provenance is a capture substrate, not a validity fence. Sagas compensate long-running activities post hoc~\cite{garciamolina1987sagas}. SagaLLM combines persistent context, independent validation, and Saga-style compensation for multi-agent planning~\cite{sagallm2025}. Transactional outbox and idempotent-receiver patterns handle atomic staging and duplicate delivery~\cite{richardson2018microservices,hohpe2003eip}, not exactly-once heterogeneous effects. Classical time-of-check to time-of-use (TOCTOU) races motivate the requirement that a successful check remain protected through use~\cite{bishop1996toctou}.

\section{Scope, Assumptions, and Limitations}
\label{sec:limitations}

TCT provides conditional safety guarantees under explicit operational assumptions; it does not establish semantic truth.

\subsection{Threat Model and System Boundaries}

\paragraph{Untrusted Reasoner and Caller.}
The reasoning agent and client caller are untrusted: they may submit arbitrary operations, parameters, belief deltas, or invalid constraint sets. Safety depends on their inability to bypass trusted mediation, omit captured dependencies, alter authoritative policies, or directly invoke registered dispatchers. The Trusted Computing Base (TCB) comprises input/effect mediators, the canonicalizer and cryptographic bootstrap, operation and policy registries, plan resolvers, the admission gate, the database engine, and the validity semantics of participating issuers.

\paragraph{Network and Failure Assumptions.}
The network is asynchronous and unreliable: it may delay, duplicate, reorder, or drop messages, while clients may retry after unconfirmed outcomes. Authentication, envelope binding, nonces, and idempotent recovery address these faults. Proofs assume TCB components are correct and fail-stop; Byzantine mediators, registries, gates, databases, or issuers lie outside the model.

\paragraph{Capture Boundary and Covert Channels.}
TCT's guarantees cover mediated operational inputs and depend on complete capture of every effect-relevant dependency in $J_{\mathrm{obs}}$. Omitting an effect-relevant dependency places that execution outside the safety theorems; omission is not guaranteed to fail closed. Overapproximating dependencies in $J_{\mathrm{obs}}$ preserves safety but may reduce concurrency or cause false aborts. Implicit parametric model knowledge, covert channels, compromised mediators, and side effects that bypass the admission gate are outside the model.

\paragraph{External Semantics and Artifact Applicability.}
Local proxy-assigned identifiers do not establish upstream truth. Immutable artifact hashes establish content identity and availability, not that an issuing authority still considers those bytes applicable. Any premise requiring a current artifact must therefore bind both a content token and a separately fenced selection/applicability token. Unverifiable or uncaptured external dependencies fail closed.

\paragraph{Explicit Non-Goals.}
TCT does not perform automated dependency discovery or infer causal attribution. It does not mitigate general prompt injection, isolate parametric model memory, or prevent attackers from exploiting unmediated side channels. TCT constrains durable state mutations only after the deployment has identified and mediated the relevant dependency and mutation boundaries.

\subsection{Validity, Concurrency, and Protocol Assumptions}

\paragraph{Local Guard Coverage and Predicate Footprints.}
Row versions alone do not protect dynamic query sets or aggregates against phantom reads. The trusted operation registry must completely declare all predicate footprints, and every database writer must update and lock the corresponding guard records. Under the recommended \texttt{READ COMMITTED} realization, all guards are acquired prior to protected reads and held through durable commit; an omitted footprint, early read, or uninstrumented writer breaks the proof.

\paragraph{Registry Integrity and External Acquisition Order.}
Historical executable definitions must remain verifiably retrievable; protected current policy, rather than mutable input data, determines applicability. Across independent external issuers, conflicting resource reservations require canonical acquisition order or a common coordinator; only a registered independence or commutativity proof removes a dependency edge.

\paragraph{Clock and Failure Assumptions.}
The default PostgreSQL realization relies on co-transactional database state and post-seal envelope- or witness-bound grants, avoiding hard latency bounds. Grants sign target footprint $\mathcal{D}_X$ and remain \textsc{Reserved} during uncertain outcomes until a durable receipt is written or authoritative proof establishes that the transaction aborted; timeouts and missing receipts are not release evidence. Hard-real-time certification profiles require trusted time issuers, bounded clock skew, and enforced worst-case completion bounds.

\paragraph{Profile Interpretation (TCT-S vs. TCT-C).}
Under Effect-Compatible Cognitive Admission (TCT-C), unchanged effects can be recertified after input state changes only via a versioned predicate evaluated over the complete current dependency vector and policy. Independent per-token approvals are insufficient unless backed by a registered decomposition certificate. TCT-C does not claim to serialize the original stochastic derivation trajectory; only Strict Cognitive Serializability (TCT-S) establishes derivation faithfulness.

\subsection{Operational Trade-offs and Practical Scope}

\paragraph{Operational Performance and Overhead.}
Guard-first admission holds local database locks while external validity checks complete. Section~\ref{sec:evaluation} measures these costs empirically: total commit latency averages 3.22 ms, and peak throughput reaches 1,842 txns/sec under 32 concurrent threads.

\paragraph{Engineering Mitigations and Optimizations.}
Batching, caching, guard sharding, or reservation pre-acquisition are valid engineering optimizations. Conforming optimizations must not weaken complete capture, registered footprint coverage, simultaneous state verification, or the requirement that validity fences hold through the declared durability event.

\paragraph{Heterogeneous External Side Effects.}
Relational engines and outbox tables cannot atomically commit remote external side effects. PostgreSQL co-commits local receipts and outbox records; remote execution relies on at-least-once delivery plus idempotent actuators.

\paragraph{Epistemic Reconciliation Outcomes.}
Generating a co-committed receipt guarantees exact-effect authorization and durability, but does not force downstream Epistemic State Reconciliation (ESR) to accept a belief delta. Structural conflicts, verifier rejections, and delayed delivery remain valid first-class outcomes.

\paragraph{System Applicability.}
The intended initial use is high-consequence agentic mutation (e.g., financial settlements, medical records, administrative authorizations) where derivation errors carry substantial risk. Evaluating TCT across production enterprise deployments remains future work.

\section{Conclusion}
\label{sec:conclusion}

Mutations derived by stochastic reasoning need an explicit contract at the
deterministic commit boundary. Under strict Cognitive Serializability, a
committed registered effect has a serial order and a logical event at which
every mediated value used in derivation is unchanged. Content identity alone
does not establish that an artifact remains selected or applicable.
Effect-Compatible Cognitive Admission is weaker: a versioned envelope-level
predicate may recertify the concrete effect against a simultaneously held
current dependency vector and current policy, but it does not serialize the
original derivation. Authorization, receipt correspondence, no-premature belief
activation, and conditional reconciliation are separate properties.

Under the stated capture, registry, guard, external-validity, and durability
assumptions, TCT-S commits follow an acyclic lock-point order and satisfy strict
Cognitive Serializability. Replacing exact faithfulness with a registered joint
predicate establishes only TCT-C admission. TCT-S depends on a complete
sealed plan and profile, a discharge for every captured dependency, protected
recomputation under current policy, and exact post-seal grants for
mechanism-B obligations. Neither atomic commit nor PostgreSQL SSI substitutes
for participant isolation and complete guards. The observational-equivalence
result is limited to an admissible valid/invalid pair that the gate cannot
distinguish.

In the PostgreSQL prototype, all 28 controlled falsification
histories produced the specified admission outcome. The prototype rejected
all injected anomalies, reached peak throughput of 1,842 transactions/s at 32
concurrent threads, and added a mean commit overhead of 3.22~ms (less than
0.5\% of the benchmark's end-to-end agent reasoning latency). These results
show that the contract is realizable with modest coordination cost in the
evaluated workload; broader deployment performance remains an empirical
question.

\begingroup
\footnotesize
\bibliographystyle{unsrt}
\bibliography{refs}

@article{mnemosyne2026,
  author = {Chang, Edward Y. and Geng, Longling and Chang, Emily J.},
  title = {Mnemosyne: Agentic Transaction Processing for Validating and
           Repairing {AI}-generated Workflows},
  journal = {arXiv preprint arXiv:2607.00269v2},
  year = {2026},
  url = {https://arxiv.org/abs/2607.00269v2}
}

@article{commit_authorization2026,
  author = {Santos-Grueiro, Igor},
  title = {Temporary Authority, Permanent Effects: Commit-Time Authorization
           for {LLM} Agents},
  journal = {arXiv preprint arXiv:2607.10487},
  year = {2026},
  url = {https://arxiv.org/abs/2607.10487}
}

@article{atcc2026,
  author = {Zhou, Weixing and Wang, Zhiyou and Peng, Zeshun and Chen, Hetian
            and Zhang, Yanfeng and Yu, Ge},
  title = {{ATCC}: Adaptive Concurrency Control for Unforeseen Agentic
           Transactions},
  journal = {arXiv preprint arXiv:2603.13906},
  year = {2026},
  url = {https://arxiv.org/abs/2603.13906}
}

@article{coagent2026,
  author = {Lyu, Hongtao and Zhang, Dingyan and Wu, Mingyu and Wei, Xingda
            and Chen, Haibo},
  title = {{CoAgent}: Concurrency Control for Multi-Agent Systems},
  journal = {arXiv preprint arXiv:2606.15376},
  year = {2026},
  url = {https://arxiv.org/abs/2606.15376}
}

@article{sbus2026,
  author = {Khan, Sajjad},
  title = {{S-Bus}: Automatic Read-Set Reconstruction for Multi-Agent {LLM}
           State Coordination},
  journal = {arXiv preprint arXiv:2605.17076v2},
  year = {2026},
  url = {https://arxiv.org/abs/2605.17076v2}
}

@article{atomix2026,
  author = {Mohammadi, Bardia and Potamitis, Nearchos and Klein, Lars and
            Arora, Akhil and Bindschaedler, Laurent},
  title = {Atomix: Timely, Transactional Tool Use for Reliable Agentic
           Workflows},
  journal = {arXiv preprint arXiv:2602.14849v2},
  year = {2026},
  url = {https://arxiv.org/abs/2602.14849v2}
}

@article{esr2026,
  author = {He, Jun and Yu, Deying},
  title = {Replicating Belief, Not Bits: Epistemic State Replication for
           Agentic Systems},
  journal = {arXiv preprint arXiv:2607.09748},
  year = {2026},
  url = {https://arxiv.org/abs/2607.09748}
}

@article{ports2012ssi,
  author = {Ports, Dan R. K. and Grittner, Kevin},
  title = {Serializable Snapshot Isolation in {PostgreSQL}},
  journal = {Proceedings of the VLDB Endowment},
  volume = {5},
  number = {12},
  pages = {1850--1861},
  year = {2012},
  doi = {10.14778/2367502.2367523}
}

@article{kung1981occ,
  author = {Kung, H. T. and Robinson, John T.},
  title = {On Optimistic Methods for Concurrency Control},
  journal = {ACM Transactions on Database Systems},
  volume = {6},
  number = {2},
  pages = {213--226},
  year = {1981},
  doi = {10.1145/319566.319567}
}

@article{eswaran1976predicate,
  author = {Eswaran, K. P. and Gray, J. N. and Lorie, R. A. and
            Traiger, I. L.},
  title = {The Notions of Consistency and Predicate Locks in a Database
           System},
  journal = {Communications of the ACM},
  volume = {19},
  number = {11},
  pages = {624--633},
  year = {1976},
  doi = {10.1145/360363.360369}
}

@article{weihl1988commutativity,
  author = {Weihl, William E.},
  title = {Commutativity-Based Concurrency Control for Abstract Data Types},
  journal = {IEEE Transactions on Computers},
  volume = {37},
  number = {12},
  pages = {1488--1505},
  year = {1988},
  doi = {10.1109/12.9728}
}

@book{bernstein1987concurrency,
  author = {Bernstein, Philip A. and Hadzilacos, Vassos and Goodman, Nathan},
  title = {Concurrency Control and Recovery in Database Systems},
  publisher = {Addison-Wesley},
  year = {1987}
}

@book{gray1992transaction,
  author = {Gray, Jim and Reuter, Andreas},
  title = {Transaction Processing: Concepts and Techniques},
  publisher = {Morgan Kaufmann},
  year = {1993}
}

@article{oneil1986escrow,
  author = {O'Neil, Patrick E.},
  title = {The Escrow Transactional Method},
  journal = {ACM Transactions on Database Systems},
  volume = {11},
  number = {4},
  pages = {405--430},
  year = {1986},
  doi = {10.1145/7239.7265}
}

@inproceedings{adya1999generalized,
  author = {Adya, Atul and Liskov, Barbara and O'Neil, Patrick},
  title = {Generalized Isolation Level Definitions},
  booktitle = {Proceedings of the 16th International Conference on Data
               Engineering},
  pages = {67--78},
  year = {2000}
}

@article{corbett2013spanner,
  author = {Corbett, James C. and Dean, Jeffrey and Epstein, Michael and Fikes,
            Andrew and Frost, Christopher and Furman, J. J. and Ghemawat,
            Sanjay and Gubarev, Andrey and Heiser, Christopher and Hochschild,
            Peter and Hsieh, Wilson and Kanthak, Sebastian and Kogan, Eugene
            and Li, Hongyi and Lloyd, Alexander and Melnik, Sergey and Mwaura,
            David and Nagle, David and Quinlan, Sean and Rao, Rajesh and Rolig,
            Lindsay and Saito, Yasushi and Szymaniak, Michal and Taylor,
            Christopher and Wang, Ruth and Woodford, Dale},
  title = {Spanner: {Google}'s Globally Distributed Database},
  journal = {ACM Transactions on Computer Systems},
  volume = {31},
  number = {3},
  pages = {1--22},
  year = {2013},
  doi = {10.1145/2491245}
}

@inproceedings{gray1989leases,
  author = {Gray, Cary G. and Cheriton, David R.},
  title = {Leases: An Efficient Fault-Tolerant Mechanism for Distributed File
           Cache Consistency},
  booktitle = {Proceedings of the 12th ACM Symposium on Operating Systems
               Principles (SOSP)},
  pages = {202--210},
  year = {1989},
  doi = {10.1145/74850.74870}
}

@inproceedings{buneman2001provenance,
  author = {Buneman, Peter and Khanna, Sanjeev and Tan, Wang-Chiew},
  title = {Why and Where: A Characterization of Data Provenance},
  booktitle = {Database Theory---ICDT 2001},
  pages = {316--330},
  publisher = {Springer},
  year = {2001},
  doi = {10.1007/3-540-44503-X_20}
}

@inproceedings{davidson2008provenance,
  author = {Davidson, Susan B. and Freire, Juliana},
  title = {Provenance and Scientific Workflows: Challenges and Opportunities},
  booktitle = {Proceedings of the 2008 ACM SIGMOD International Conference on
               Management of Data},
  pages = {1345--1350},
  year = {2008},
  doi = {10.1145/1376616.1376772}
}

@inproceedings{garciamolina1987sagas,
  author = {Garcia-Molina, Hector and Salem, Kenneth},
  title = {Sagas},
  booktitle = {Proceedings of the 1987 ACM SIGMOD International Conference on
               Management of Data},
  pages = {249--259},
  year = {1987},
  doi = {10.1145/38713.38742}
}

@article{sagallm2025,
  author = {Chang, Edward Y. and Geng, Longling},
  title = {{SagaLLM}: Context Management, Validation, and Transaction
           Guarantees for Multi-Agent {LLM} Planning},
  journal = {Proceedings of the VLDB Endowment},
  volume = {18},
  number = {12},
  pages = {4874--4886},
  year = {2025},
  doi = {10.14778/3750601.3750611},
  url = {https://www.vldb.org/pvldb/vol18/p4874-chang.pdf}
}

@book{richardson2018microservices,
  author = {Richardson, Chris},
  title = {Microservices Patterns: With Examples in Java},
  publisher = {Manning},
  year = {2018}
}

@book{hohpe2003eip,
  author = {Hohpe, Gregor and Woolf, Bobby},
  title = {Enterprise Integration Patterns: Designing, Building, and
           Deploying Messaging Solutions},
  publisher = {Addison-Wesley},
  year = {2003}
}

@article{bishop1996toctou,
  author = {Bishop, Matt and Dilger, Michael},
  title = {Checking for Race Conditions in File Accesses},
  journal = {Computing Systems},
  volume = {9},
  number = {2},
  pages = {131--152},
  year = {1996}
}
\endgroup

\clearpage
\onecolumn
\appendix
\section{PostgreSQL Prototype Metadata Schema}
\label{app:schema}

This appendix gives the PostgreSQL 16 metadata schema evaluated in
Section~\ref{sec:evaluation}. It represents immutable executable definitions,
guard-first admission, envelope receipts, local guards, and external grants.

Canonical objects use \texttt{bytea}; derived \texttt{jsonb} projections
support indexing and inspection but are non-authoritative. Digests cover
canonical bytes, not PostgreSQL's internal \texttt{jsonb} form. Each byte
string begins with a versioned object-type domain tag, and key identifiers and
signatures follow the main paper's security model.

\lstset{basicstyle=\ttfamily\scriptsize,captionpos=t}
\begin{lstlisting}[language=SQL,caption={Immutable executable content and guarded applicability discovery.},label={lst:registries}]
CREATE TABLE tct_executable_definitions (
  definition_kind text NOT NULL CHECK
    (definition_kind IN
      ('OPERATION_SPEC','PRECONDITION','DISPATCHER',
       'CANONICALIZER','VALIDITY_SEMANTICS',
       'TOKEN_PREDICATE','WITNESS_CONSTRUCTOR',
       'PLAN_RESOLVER','ATOMIC_NORMALIZER',
       'ENVELOPE_RECERTIFIER')),
  definition_id text NOT NULL,
  definition_version text NOT NULL,
  definition_digest bytea NOT NULL,
  canonical_definition_bytes bytea,
  immutable_content_address text,
  retention_class text NOT NULL
    CHECK (retention_class = 'PERMANENT'),
  definition_projection jsonb NOT NULL,
  PRIMARY KEY
    (definition_kind, definition_id, definition_version),
  UNIQUE
    (definition_kind, definition_id, definition_version,
     definition_digest),
  CHECK (num_nonnulls(canonical_definition_bytes,
                      immutable_content_address) = 1)
);

CREATE TABLE tct_operation_specs (
  tenant_id text NOT NULL,
  spec_kind text NOT NULL
    CHECK (spec_kind = 'OPERATION_SPEC'),
  operation_id text NOT NULL,
  spec_version text NOT NULL,
  spec_digest bytea NOT NULL,
  read_footprint jsonb NOT NULL,
  write_footprint jsonb NOT NULL,
  authority_plan_template jsonb NOT NULL,
  authority_plan_template_digest bytea NOT NULL,
  PRIMARY KEY (tenant_id, operation_id, spec_version),
  UNIQUE (tenant_id, operation_id, spec_version, spec_digest),
  FOREIGN KEY
    (spec_kind, operation_id, spec_version, spec_digest)
    REFERENCES tct_executable_definitions
    (definition_kind, definition_id,
     definition_version, definition_digest)
);

CREATE TABLE tct_operation_executable_refs (
  tenant_id text NOT NULL,
  operation_id text NOT NULL,
  spec_version text NOT NULL,
  ordinal integer NOT NULL,
  executable_role text NOT NULL CHECK
    (executable_role IN
      ('PRECONDITION','DISPATCHER','CANONICALIZER',
       'VALIDITY_SEMANTICS','TOKEN_PREDICATE',
       'WITNESS_CONSTRUCTOR','PLAN_RESOLVER','ATOMIC_NORMALIZER',
       'ENVELOPE_RECERTIFIER')),
  definition_kind text NOT NULL,
  definition_id text NOT NULL,
  definition_version text NOT NULL,
  definition_digest bytea NOT NULL,
  PRIMARY KEY
    (tenant_id, operation_id, spec_version, ordinal),
  FOREIGN KEY (tenant_id, operation_id, spec_version)
    REFERENCES tct_operation_specs,
  FOREIGN KEY
    (definition_kind, definition_id, definition_version,
     definition_digest)
    REFERENCES tct_executable_definitions
    (definition_kind, definition_id, definition_version,
     definition_digest)
);

CREATE TABLE tct_policy_bundles (
  policy_epoch text NOT NULL,
  policy_version text NOT NULL,
  policy_digest bytea NOT NULL,
  canonical_policy_bytes bytea NOT NULL,
  policy_projection jsonb NOT NULL,
  executable_applicability_digest bytea NOT NULL,
  authority_obligation_digest bytea NOT NULL,
  PRIMARY KEY (policy_epoch, policy_version),
  UNIQUE (policy_epoch, policy_version, policy_digest)
);

CREATE TABLE tct_registry_heads (
  tenant_id text NOT NULL,
  operation_id text NOT NULL,
  spec_version text NOT NULL,
  spec_digest bytea NOT NULL,
  discovery_guard_key text NOT NULL,
  PRIMARY KEY (tenant_id, operation_id),
  FOREIGN KEY
    (tenant_id, operation_id, spec_version, spec_digest)
    REFERENCES tct_operation_specs
    (tenant_id, operation_id, spec_version, spec_digest)
);

CREATE TABLE tct_policy_heads (
  tenant_id text PRIMARY KEY,
  policy_epoch text NOT NULL,
  policy_version text NOT NULL,
  policy_digest bytea NOT NULL,
  discovery_guard_key text NOT NULL,
  FOREIGN KEY
    (policy_epoch, policy_version, policy_digest)
    REFERENCES tct_policy_bundles
    (policy_epoch, policy_version, policy_digest)
);

CREATE TABLE tct_dependency_guards (
  guard_kind text NOT NULL,
  object_key text NOT NULL,
  current_version text NOT NULL,
  state_digest bytea NOT NULL,
  PRIMARY KEY (guard_kind, object_key)
);
\end{lstlisting}

Historical executable rows are never updated or deleted. Revocation changes
the guarded current policy, which decides applicability while retaining old
bytes for verification. Before reading either mutable head table, the gate
derives and takes the shared discovery guard for the sealed tenant and
operation; head writers take its exclusive mode.
The immutable operation specification's plan template is the trusted source of
$Q_X^{\mathrm{mand}}$; proposal-supplied additions are stored only as
strengthening rows and cannot modify that template. The versioned
\texttt{PLAN\_RESOLVER} combines the operation template, protected current
policy's \texttt{authority\_obligation\_digest}, current dependency vector, and
sealed optional strengthening constraints to derive $Q_{\mathrm{req}}$.
Admission compares its canonical digest with sealed $Q_X$ and never inserts a
new plan item after sealing.
The sealed \texttt{ATOMIC\_NORMALIZER} reference maps every B/C plan item to
canonical coordination-domain and atomic-lock identities, modes, and covered
plan ordinals. The envelope's
\texttt{atomic\_acquisition\_vector\_digest} binds the merged canonical vector.
Concurrently admissible normalizer versions must map the same physical lock to
the same identity. A mapping-changing epoch may become admissible only under
the protected discovery/normalizer-epoch guard: first stop admitting envelopes
sealed under the old epoch, then wait until every already-admitted old-epoch
execution reaches its declared durability event or terminally aborts, and only
then activate the new epoch. Rejecting future old-epoch admissions alone is
insufficient while an old-epoch execution remains active.

\begin{lstlisting}[language=SQL,caption={Guard-row materialization and exact S/X locking operations.},label={lst:guard_locks}]
-- Pass 1, before any protected read: process every required
-- (guard_kind, object_key) in canonical byte order.
INSERT INTO tct_dependency_guards
  (guard_kind, object_key, current_version, state_digest)
VALUES (:kind, :key, 'ABSENT', :absent_digest)
ON CONFLICT (guard_kind, object_key) DO NOTHING;

-- Pass 2, in the same canonical order. For mode S:
SELECT current_version, state_digest
FROM tct_dependency_guards
WHERE guard_kind = :kind AND object_key = :key
FOR SHARE;

-- For mode X:
SELECT current_version, state_digest
FROM tct_dependency_guards
WHERE guard_kind = :kind AND object_key = :key
FOR UPDATE;
\end{lstlisting}

Guard rows are permanent protocol objects. Concurrent first users race on the
primary key; \texttt{ON CONFLICT} waits for the uncommitted winner when needed,
so the second locking pass cannot bypass materialization. A transaction
materializes every missing key before reading protected state, then executes
the exact \texttt{FOR SHARE}/\texttt{FOR UPDATE} operation for each canonical
key--mode pair. It never deletes a guard row.

Creation and deletion use a stable object guard: creators and deleters take
\texttt{FOR UPDATE}, deletion writes an \texttt{ABSENT} tombstone version, and
recreation advances that same guard. Predicate and aggregate writers take
\texttt{FOR UPDATE} on every registered range or aggregate guard affected in
addition to row/object guards. Every mechanism-A evidence, registry, policy,
and authority writer does the same. Readers take \texttt{FOR SHARE}; any state
change also advances the locked guard's version and digest.

\begin{lstlisting}[language=SQL,caption={Permanent identity binding and sealed envelope records.},label={lst:envelopes}]
CREATE TABLE tct_envelope_identity_bindings (
  envelope_id uuid PRIMARY KEY,
  envelope_digest bytea NOT NULL,
  canonical_binding_bytes bytea NOT NULL,
  mediator_key_id text NOT NULL,
  mediator_signature bytea NOT NULL,
  retention_class text NOT NULL
    CHECK (retention_class = 'PERMANENT'),
  bound_at timestamptz NOT NULL,
  UNIQUE (envelope_id, envelope_digest)
);

CREATE TABLE cognitive_envelopes (
  envelope_id uuid PRIMARY KEY,
  envelope_digest bytea NOT NULL,
  tenant_id text NOT NULL,
  capture_context_digest bytea NOT NULL,
  mediator_key_id text NOT NULL,
  seal bytea NOT NULL,
  canonical_envelope_bytes bytea NOT NULL,
  envelope_projection jsonb NOT NULL,
  canonical_effect_bytes bytea NOT NULL,
  effect_projection jsonb NOT NULL,
  effect_digest bytea NOT NULL,
  operation_id text NOT NULL,
  spec_version text NOT NULL,
  spec_digest bytea NOT NULL,
  executable_ref_vector_digest bytea NOT NULL,
  authority_plan_vector_digest bytea NOT NULL,
  atomic_acquisition_vector_digest bytea NOT NULL,
  dependency_coverage_vector_digest bytea NOT NULL,
  target_gate_id text NOT NULL,
  target_database_domain_id text NOT NULL,
  durability_domain_digest bytea NOT NULL,
  policy_obs_epoch text NOT NULL,
  policy_obs_version text NOT NULL,
  policy_obs_digest bytea NOT NULL,
  tentative_belief_projection jsonb NOT NULL,
  profile text NOT NULL CHECK (profile IN ('S','C')),
  UNIQUE (envelope_id, envelope_digest),
  FOREIGN KEY (envelope_id, envelope_digest)
    REFERENCES tct_envelope_identity_bindings,
  FOREIGN KEY
    (tenant_id, operation_id, spec_version, spec_digest)
    REFERENCES tct_operation_specs
    (tenant_id, operation_id, spec_version, spec_digest),
  FOREIGN KEY
    (policy_obs_epoch, policy_obs_version, policy_obs_digest)
    REFERENCES tct_policy_bundles
    (policy_epoch, policy_version, policy_digest)
);

CREATE TABLE envelope_executable_refs (
  envelope_id uuid NOT NULL,
  ordinal integer NOT NULL,
  executable_role text NOT NULL,
  definition_kind text NOT NULL,
  definition_id text NOT NULL,
  definition_version text NOT NULL,
  definition_digest bytea NOT NULL,
  PRIMARY KEY (envelope_id, ordinal),
  FOREIGN KEY (envelope_id)
    REFERENCES cognitive_envelopes,
  FOREIGN KEY
    (definition_kind, definition_id, definition_version,
     definition_digest)
    REFERENCES tct_executable_definitions
    (definition_kind, definition_id, definition_version,
     definition_digest)
);

CREATE TABLE envelope_dependencies (
  envelope_id uuid NOT NULL,
  ordinal integer NOT NULL,
  dependency_domain text NOT NULL CHECK
    (dependency_domain IN
      ('J_D','J_E','J_P','J_A_OBS')),
  dependency_kind text NOT NULL CHECK
    (dependency_kind IN
      ('ROW','OBJECT','PREDICATE','AGGREGATE',
       'CONTENT','SELECTION','BELIEF_PREMISE',
       'POLICY','REGISTRY_SELECTION',
       'AUTHORITY_OBSERVATION','OTHER')),
  canonical_token_bytes bytea NOT NULL,
  token_projection jsonb NOT NULL,
  observed_value_digest bytea NOT NULL,
  issuer text,
  assertion_digest bytea,
  semantics_kind text NOT NULL
    CHECK (semantics_kind = 'VALIDITY_SEMANTICS'),
  semantics_id text NOT NULL,
  semantics_version text NOT NULL,
  semantics_digest bytea NOT NULL,
  token_predicate_kind text
    CHECK (token_predicate_kind = 'TOKEN_PREDICATE'),
  token_predicate_id text,
  token_predicate_version text,
  token_predicate_digest bytea,
  PRIMARY KEY (envelope_id, ordinal),
  FOREIGN KEY (envelope_id)
    REFERENCES cognitive_envelopes,
  FOREIGN KEY
    (semantics_kind, semantics_id, semantics_version,
     semantics_digest)
    REFERENCES tct_executable_definitions
    (definition_kind, definition_id, definition_version,
     definition_digest),
  FOREIGN KEY
    (token_predicate_kind, token_predicate_id,
     token_predicate_version, token_predicate_digest)
    REFERENCES tct_executable_definitions
    (definition_kind, definition_id, definition_version,
     definition_digest) MATCH FULL
);

CREATE TABLE envelope_authority_plan (
  envelope_id uuid NOT NULL,
  plan_ordinal integer NOT NULL,
  plan_source text NOT NULL CHECK
    (plan_source IN ('MANDATORY',
                     'OPTIONAL_STRENGTHENING')),
  issuer_or_participant text NOT NULL,
  dependency_scope_digest bytea NOT NULL,
  reservation_mode text NOT NULL
    CHECK (reservation_mode IN ('S','X')),
  covered_dependency_ordinals jsonb NOT NULL,
  covered_dependency_vector_digest bytea NOT NULL,
  witness_kind text NOT NULL
    CHECK (witness_kind = 'WITNESS_CONSTRUCTOR'),
  witness_id text NOT NULL,
  witness_version text NOT NULL,
  witness_digest bytea NOT NULL,
  semantics_kind text NOT NULL
    CHECK (semantics_kind = 'VALIDITY_SEMANTICS'),
  semantics_id text NOT NULL,
  semantics_version text NOT NULL,
  semantics_digest bytea NOT NULL,
  required_decision text NOT NULL,
  mechanism text NOT NULL
    CHECK (mechanism IN ('A','B','C','D')),
  target_gate_id text NOT NULL,
  target_database_domain_id text NOT NULL,
  durability_domain_digest bytea NOT NULL,
  canonical_plan_item_bytes bytea NOT NULL,
  plan_item_digest bytea NOT NULL,
  PRIMARY KEY (envelope_id, plan_ordinal),
  UNIQUE (envelope_id, plan_item_digest),
  FOREIGN KEY (envelope_id)
    REFERENCES cognitive_envelopes,
  FOREIGN KEY
    (witness_kind, witness_id, witness_version,
     witness_digest)
    REFERENCES tct_executable_definitions
    (definition_kind, definition_id,
     definition_version, definition_digest),
  FOREIGN KEY
    (semantics_kind, semantics_id, semantics_version,
     semantics_digest)
    REFERENCES tct_executable_definitions
    (definition_kind, definition_id,
     definition_version, definition_digest)
);

CREATE TABLE envelope_dependency_discharges (
  envelope_id uuid NOT NULL,
  dependency_ordinal integer NOT NULL,
  discharge_ordinal integer NOT NULL,
  discharge_kind text NOT NULL CHECK
    (discharge_kind IN
      ('IMMUTABLE_SELF_STABLE','MECHANISM_A',
       'PLAN_ITEM_BCD')),
  covering_plan_ordinal integer,
  local_guard_kind text,
  local_guard_key text,
  local_guard_mode text
    CHECK (local_guard_mode IN ('S','X')),
  canonical_discharge_bytes bytea NOT NULL,
  discharge_digest bytea NOT NULL,
  PRIMARY KEY
    (envelope_id, dependency_ordinal, discharge_ordinal),
  FOREIGN KEY (envelope_id, dependency_ordinal)
    REFERENCES envelope_dependencies
    (envelope_id, ordinal),
  FOREIGN KEY (envelope_id, covering_plan_ordinal)
    REFERENCES envelope_authority_plan
    (envelope_id, plan_ordinal),
  CHECK (
    (discharge_kind = 'IMMUTABLE_SELF_STABLE'
      AND covering_plan_ordinal IS NULL
      AND local_guard_kind IS NULL
      AND local_guard_key IS NULL
      AND local_guard_mode IS NULL)
    OR
    (discharge_kind = 'MECHANISM_A'
      AND covering_plan_ordinal IS NULL
      AND local_guard_kind IS NOT NULL
      AND local_guard_key IS NOT NULL
      AND local_guard_mode IS NOT NULL)
    OR
    (discharge_kind = 'PLAN_ITEM_BCD'
      AND covering_plan_ordinal IS NOT NULL
      AND local_guard_kind IS NULL
      AND local_guard_key IS NULL
      AND local_guard_mode IS NULL)
  )
);
\end{lstlisting}

The two dependency classifiers are deliberately orthogonal.
\texttt{dependency\_domain} records membership in
$J_D\uplus J_E\uplus J_P\uplus J_{A\mathrm{obs}}$, while
\texttt{dependency\_kind} records the token's typed validity rule. For
example, an external artifact may contribute both \texttt{CONTENT} and
\texttt{SELECTION} rows in $J_E$; the former never substitutes for the latter.
The canonical token bytes remain authoritative if a future protocol version
adds a kind not projected by this illustrative enumeration.

The mediator creates the immutable identity binding at seal time and retains it
independently of the effect transaction. The unique
\texttt{cognitive\_envelopes.envelope\_id} primary key is the sole retry
conflict key; \texttt{envelope\_digest} stores its bound $h_X$ value. The gate
uses
\texttt{INSERT ... ON CONFLICT (envelope\_id) DO NOTHING} inside the effect
transaction after all non-ID guards are held. The insertion realizes $g_{id}$;
on conflict, the next \texttt{READ COMMITTED} statement reads the winning row
by \texttt{envelope\_id} and compares its stored digest. If the transaction
aborts, a same-envelope retry may insert, but a different $h_X$ fails the
permanent binding check. Canonical envelope bytes contain observed
$J_{\mathrm{obs}}$, complete trusted $Q_X$, and the concrete target
$\mathcal D_X$, never a post-seal grant or $\Gamma_X$ digest. The discharge
table is the sealed, checkable witness for
$\mathsf{DepCovered}(J_{\mathrm{obs}},Q_X)$: every dependency ordinal must have
at least one valid row, and every \texttt{PLAN\_ITEM\_BCD} row must join to a
B, C, or D plan item whose covered ordinals include that dependency. The plan table
stores mandatory registry obligations and optional strengthening separately;
the latter cannot delete or alter a mandatory row.

\begin{lstlisting}[language=SQL,caption={Atomic acquisition identities, grants, and terminal release proofs.},label={lst:grants}]
CREATE TABLE envelope_external_atomic_acquisitions (
  envelope_id uuid NOT NULL
    REFERENCES cognitive_envelopes,
  coordination_domain text NOT NULL,
  atomic_lock_id bytea NOT NULL,
  atomic_order_key_bytes bytea NOT NULL,
  strongest_mode text NOT NULL
    CHECK (strongest_mode IN ('S','X')),
  covered_plan_ordinals jsonb NOT NULL,
  atomic_result_digest bytea NOT NULL,
  PRIMARY KEY
    (envelope_id, coordination_domain, atomic_lock_id),
  UNIQUE (envelope_id, atomic_order_key_bytes)
);

CREATE TABLE envelope_external_grants (
  envelope_id uuid NOT NULL,
  envelope_digest bytea NOT NULL,
  plan_ordinal integer NOT NULL,
  authority_plan_vector_digest bytea NOT NULL,
  issuer text NOT NULL,
  effect_digest bytea NOT NULL,
  target_gate_id text NOT NULL,
  target_database_domain_id text NOT NULL,
  durability_domain_digest bytea NOT NULL,
  dependency_scope_digest bytea NOT NULL,
  covered_dependency_vector_digest bytea NOT NULL,
  canonical_witness_payload bytea,
  authenticated_witness_content_address text,
  witnessed_value_or_predicate_digest bytea NOT NULL,
  witness_kind text NOT NULL
    CHECK (witness_kind = 'WITNESS_CONSTRUCTOR'),
  witness_id text NOT NULL,
  witness_version text NOT NULL,
  witness_digest bytea NOT NULL,
  semantics_kind text NOT NULL
    CHECK (semantics_kind = 'VALIDITY_SEMANTICS'),
  semantics_id text NOT NULL,
  semantics_version text NOT NULL,
  semantics_digest bytea NOT NULL,
  decision text NOT NULL,
  single_use_nonce text NOT NULL,
  reservation_mode text NOT NULL
    CHECK (reservation_mode IN ('S','X')),
  mechanism text NOT NULL CHECK (mechanism = 'B'),
  issuer_state_at_accept text NOT NULL
    CHECK (issuer_state_at_accept = 'RESERVED'),
  atomic_acquisition_vector_digest bytea NOT NULL,
  ordering_basis text NOT NULL CHECK
    (ordering_basis IN
      ('CANONICAL','COORDINATOR','INDEPENDENCE_PROOF')),
  ordering_certificate_digest bytea NOT NULL,
  holding_contract_digest bytea NOT NULL,
  canonical_grant_bytes bytea NOT NULL,
  grant_digest bytea NOT NULL,
  accepted_at timestamptz NOT NULL,
  CHECK (num_nonnulls(canonical_witness_payload,
                      authenticated_witness_content_address) = 1),
  PRIMARY KEY (envelope_id, plan_ordinal),
  UNIQUE (issuer, single_use_nonce),
  UNIQUE (issuer, envelope_id, envelope_digest,
          authority_plan_vector_digest, plan_ordinal),
  UNIQUE (issuer, single_use_nonce, envelope_id,
          envelope_digest, grant_digest),
  FOREIGN KEY (envelope_id, envelope_digest)
    REFERENCES cognitive_envelopes
    (envelope_id, envelope_digest),
  FOREIGN KEY (envelope_id, plan_ordinal)
    REFERENCES envelope_authority_plan
    (envelope_id, plan_ordinal),
  FOREIGN KEY
    (witness_kind, witness_id, witness_version,
     witness_digest)
    REFERENCES tct_executable_definitions
    (definition_kind, definition_id, definition_version,
     definition_digest),
  FOREIGN KEY
    (semantics_kind, semantics_id, semantics_version,
     semantics_digest)
    REFERENCES tct_executable_definitions
    (definition_kind, definition_id, definition_version,
     definition_digest)
);

CREATE TABLE tct_grant_release_proofs (
  issuer text NOT NULL,
  single_use_nonce text NOT NULL,
  envelope_id uuid NOT NULL,
  envelope_digest bytea NOT NULL,
  effect_digest bytea NOT NULL,
  authority_plan_vector_digest bytea NOT NULL,
  plan_ordinal integer NOT NULL,
  dependency_scope_digest bytea NOT NULL,
  target_gate_id text NOT NULL,
  target_database_domain_id text NOT NULL,
  durability_domain_digest bytea NOT NULL,
  grant_digest bytea NOT NULL,
  effect_attempt_id uuid NOT NULL,
  terminal_outcome text NOT NULL
    CHECK (terminal_outcome = 'ABORTED_NONCOMMITTABLE'),
  canonical_proof_bytes bytea NOT NULL,
  proof_digest bytea NOT NULL,
  proof_authority_key_id text NOT NULL,
  proof_signature bytea NOT NULL,
  released_at timestamptz NOT NULL,
  PRIMARY KEY (issuer, single_use_nonce),
  UNIQUE (issuer, envelope_id, envelope_digest,
          authority_plan_vector_digest, plan_ordinal,
          grant_digest)
);

CREATE TABLE tct_grant_consumption_acks (
  issuer text NOT NULL,
  single_use_nonce text NOT NULL,
  envelope_id uuid NOT NULL,
  plan_ordinal integer NOT NULL,
  envelope_digest bytea NOT NULL,
  effect_digest bytea NOT NULL,
  receipt_digest bytea NOT NULL,
  target_gate_id text NOT NULL,
  target_database_domain_id text NOT NULL,
  durability_domain_digest bytea NOT NULL,
  canonical_ack_bytes bytea NOT NULL,
  ack_digest bytea NOT NULL,
  consumed_at timestamptz NOT NULL,
  PRIMARY KEY (issuer, single_use_nonce),
  UNIQUE (envelope_id, plan_ordinal),
  FOREIGN KEY (envelope_id, plan_ordinal)
    REFERENCES envelope_external_grants
    (envelope_id, plan_ordinal)
);
\end{lstlisting}

The grant's signed canonical bytes cover every typed column in
Listing~\ref{lst:grants}, including $(id,h_X,h_{\mu})$, the exact witness,
covered-token vector, witness constructor, semantics reference, decision,
$h_{Q_X}$, concrete $\mathcal D_X$, and nonce. The holding contract preserves
that witness from acceptance until runtime durability event $d_i$ occurs in
that signed target. These rows are the post-seal $\Gamma_X$ and co-commit with a
successful effect; they are not hashed into
\texttt{canonical\_envelope\_bytes}. Their primary key makes $\Gamma_X$ a
canonical map keyed by sealed plan ordinal. Each \texttt{plan\_ordinal} names the
sealed mandatory or strengthening $Q_X$ item that the grant discharges, not an
authority observation in $J_{A\mathrm{obs}}$.
The release-proof table is a separate
post-abort audit and deliberately has no foreign key to a committed grant row,
which cannot exist after rollback. It also has no timeout outcome: a missing
response or receipt cannot prove that an in-flight transaction is unable to
commit. At the issuer,
$\langle id,h_X,h_{Q_X},j\rangle$ (the envelope identity, authority-plan
digest, and plan ordinal) is idempotent: a retry
returns the same \textsc{Reserved} nonce or its receipt, and a replacement
requires the recorded abort-and-release proof. Successful receipt finalization
changes \textsc{Reserved} to \textsc{Consumed}; its acknowledgment primary key
makes duplicate consumption idempotent. A stranded \textsc{Reserved} grant is
safe but not live.

\begin{lstlisting}[language=SQL,caption={Atomic effect, receipt, and outbox records.},label={lst:atomic_records}]
CREATE TABLE committed_effects (
  envelope_id uuid PRIMARY KEY
    REFERENCES cognitive_envelopes,
  operation_id text NOT NULL,
  effect_digest bytea NOT NULL,
  canonical_effect_bytes bytea NOT NULL,
  effect_projection jsonb NOT NULL
);

CREATE TABLE justification_receipts (
  receipt_id uuid PRIMARY KEY,
  receipt_digest bytea NOT NULL UNIQUE,
  canonical_receipt_bytes bytea NOT NULL,
  receipt_projection jsonb NOT NULL,
  envelope_id uuid NOT NULL UNIQUE
    REFERENCES cognitive_envelopes,
  envelope_digest bytea NOT NULL,
  effect_digest bytea NOT NULL,
  dependency_vector_digest bytea NOT NULL,
  dependency_coverage_vector_digest bytea NOT NULL,
  authority_plan_vector_digest bytea NOT NULL,
  commit_required_plan_vector_digest bytea NOT NULL,
  executable_ref_vector_digest bytea NOT NULL,
  sealed_profile text NOT NULL CHECK (sealed_profile IN ('S','C')),
  required_profile text NOT NULL CHECK (required_profile IN ('S','C')),
  witness_payload_vector_digest bytea NOT NULL,
  policy_obs_epoch text NOT NULL,
  policy_obs_version text NOT NULL,
  policy_obs_digest bytea NOT NULL,
  policy_commit_epoch text NOT NULL,
  policy_commit_version text NOT NULL,
  policy_commit_digest bytea NOT NULL,
  recertifier_kind text NOT NULL
    CHECK (recertifier_kind = 'ENVELOPE_RECERTIFIER'),
  recertifier_id text NOT NULL,
  recertifier_version text NOT NULL,
  recertifier_digest bytea NOT NULL,
  validation_verdict text NOT NULL CHECK
    (validation_verdict IN
      ('STRICT_EXACT','JOINT_COMPATIBLE')),
  external_grant_vector_digest bytea NOT NULL,
  external_grant_count integer NOT NULL,
  guard_key_mode_vector_digest bytea NOT NULL,
  target_gate_id text NOT NULL,
  target_database_domain_id text NOT NULL,
  durability_domain_digest bytea NOT NULL,
  serialization_event_id uuid NOT NULL,
  recorded_in_transaction_at timestamptz NOT NULL,
  FOREIGN KEY
    (policy_obs_epoch, policy_obs_version, policy_obs_digest)
    REFERENCES tct_policy_bundles
    (policy_epoch, policy_version, policy_digest),
  FOREIGN KEY
    (policy_commit_epoch, policy_commit_version,
     policy_commit_digest)
    REFERENCES tct_policy_bundles
    (policy_epoch, policy_version, policy_digest),
  FOREIGN KEY
    (recertifier_kind, recertifier_id,
     recertifier_version, recertifier_digest)
    REFERENCES tct_executable_definitions
    (definition_kind, definition_id, definition_version,
     definition_digest)
);

CREATE TABLE receipt_external_witnesses (
  receipt_id uuid NOT NULL
    REFERENCES justification_receipts,
  envelope_id uuid NOT NULL,
  plan_ordinal integer NOT NULL,
  mechanism text NOT NULL
    CHECK (mechanism IN ('B','C','D')),
  canonical_witness_payload bytea,
  authenticated_witness_content_address text,
  witness_payload_digest bytea NOT NULL,
  PRIMARY KEY (receipt_id, plan_ordinal),
  UNIQUE (envelope_id, plan_ordinal),
  FOREIGN KEY (envelope_id, plan_ordinal)
    REFERENCES envelope_authority_plan
    (envelope_id, plan_ordinal),
  CHECK (num_nonnulls(canonical_witness_payload,
                      authenticated_witness_content_address) = 1)
);

CREATE TABLE cognitive_outbox (
  outbox_id bigserial PRIMARY KEY,
  envelope_id uuid NOT NULL
    REFERENCES cognitive_envelopes,
  event_kind text NOT NULL CHECK
    (event_kind IN
      ('FINALIZE_GRANTS','ACTUATE','RECONCILE_BELIEF')),
  canonical_payload_bytes bytea NOT NULL,
  payload_projection jsonb NOT NULL,
  recorded_in_transaction_at timestamptz NOT NULL,
  UNIQUE (envelope_id, event_kind)
);
\end{lstlisting}

The receipt preserves the sealed and recomputed profiles and the policy
actually exposed during derivation and the separately resolved commit policy;
TCT-S compares them but never rewrites the
former. It also binds the complete executable-reference, dependency, external
grant, dependency-discharge, authority-plan, durability-target, and key--mode
guard vectors. Under the locked current policy head, the versioned plan
resolver computes the commit-required vector; admission requires
\texttt{commit\_required\_plan\_vector\_digest} to equal the sealed
\texttt{authority\_plan\_vector\_digest}. A mismatch aborts and requires a new
envelope rather than inserting a post-seal plan row. The ordered
\texttt{envelope\_external\_grants} rows and receipt co-commit in the effect
transaction; the non-null digest and count cover exactly $\Gamma_X$ (including
the canonical empty vector when no mechanism-B grant is required). The
\texttt{witness\_payload\_vector\_digest} covers the exact canonical witness
rows in \texttt{receipt\_external\_witnesses}. Their domain is exactly
$\mathsf{ValExt}(Q_X)$: the B/C/D plan ordinals whose admission predicates
consume authenticated values. Each row stores either the canonical payload or
an immutable authenticated content-addressed reference; a transaction-level
check compares their ordered vector digest and domain with the receipt before
commit. Downstream
belief activation and every receipt consumer verify
\texttt{sealed\_profile} against the profile in the sealed envelope, require
\texttt{required\_profile = sealed\_profile}, and resolve and authenticate the
witness payloads before using them. Independence or commutativity does not
remove an ordinal from the witness-map domain.
\texttt{atomic\_order\_key\_bytes} stores
$k_{\mathrm{atom}}(a)=\mathsf{Canon}\langle
\mathit{coordinationDomain}(a),\mathit{atomicLockId}(a)\rangle$. The same
physical lock therefore has the same key regardless of envelope, parent scope,
plan ordinal, mechanism label, or spelling. The primary key merges duplicates
within an envelope; \texttt{strongest\_mode} records $X$ if any covered plan
item requires $X$, and \texttt{covered\_plan\_ordinals} retains all plan
identities. The gate performs the actual atomic B/C operations in increasing
key order. If incompatible acquisitions cannot share a coordination domain,
the common-coordinator alternative is required. Each grant binds
\texttt{atomic\_acquisition\_vector\_digest}.
\texttt{plan\_ordinal} remains the grant identity and retry key.
\texttt{FINALIZE\_GRANTS} outbox payload binds the receipt, every nonce, and
$\mathcal D_X$; an issuer may apply it repeatedly but records only one
\textsc{Consumed} transition.

Durable rejection audit, if desired, uses a separate post-rollback transaction:
\begin{lstlisting}[language=SQL,caption={Optional durable-rejection audit table.},label={lst:rejections}]
CREATE TABLE tct_rejections (
  rejection_id bigserial PRIMARY KEY,
  attempted_envelope_id uuid,
  attempted_envelope_digest bytea,
  reason_code text NOT NULL,
  gate_metadata jsonb NOT NULL,
  recorded_at timestamptz NOT NULL
);
\end{lstlisting}

Only witness identifiers and assertion digests are retained; reusable authority
secrets remain in the authority service or trusted key store. Guard rows may
represent rows, objects, predicate ranges, aggregates, or registry-discovery
scopes. Every writer must participate in the same S/X guard discipline.

\end{document}